\documentclass[journal]{IEEEtran}

\ifCLASSINFOpdf
\else
   \usepackage[dvips]{graphicx}
\fi
\usepackage{url}

\usepackage{amsmath}
\usepackage{amssymb}
\usepackage{latexsym}
\usepackage{verbatim}
\usepackage{subfigure}
\usepackage{psfrag}
\usepackage{hyperref}

\usepackage{tcolorbox}
\definecolor{darkred}{RGB}{250,0,0}
\definecolor{darkgreen}{RGB}{0,150,0}
\definecolor{myblue}{RGB}{0,0,250}
\definecolor{darkblue}{RGB}{0,0,200}
\hypersetup{colorlinks=true, linkcolor=darkred, citecolor=myblue, urlcolor=darkblue}
\usepackage{amsthm}

\newtheorem{theorem}{Theorem}

\newtheorem{lemma}{Lemma}
\newtheorem{corollary}{Corollary}

\newtheorem{assumption}{Assumption}

{\begin{list}               
    {$\bullet$ \hfill}{
        \setlength{\leftmargin}{\parindent}
        \setlength{\parsep}{0.04\baselineskip}
        \setlength{\itemsep}{0.5\parsep}
        \setlength{\labelwidth}{\leftmargin}
        \setlength{\labelsep}{0em}}
    }
{\end{list}}

\providecommand{\cref}[1]{Chapter~\ref{chap:#1}}

\providecommand{\R}{\ensuremath{\mathbb{R}}}

\providecommand{\abs}[1]{\lvert#1\rvert}
\providecommand{\norm}[1]{\lVert#1\rVert}

\renewcommand{\vec}[1]{\ensuremath{\boldsymbol{#1}}}
\providecommand{\mat}[1]{\ensuremath{\boldsymbol{#1}}}

\providecommand{\calM}{\mathcal{M}}

 \providecommand{\mB}{\mat{B}}
\providecommand{\mC}{\mat{C}}

\providecommand{\mI}{\mat{I}}  
\providecommand{\mK}{\mat{K}} \providecommand{\mL}{\mat{L}} 
 \providecommand{\mP}{\mat{P}}

\providecommand{\mGm}{\mat{\Gamma}} \providecommand{\mG}{\mat{G}}
\providecommand{\mSig}{\mat{\Sigma}}

\providecommand{\va}{\vec{a}} \providecommand{\vb}{\vec{b}}
 
\providecommand{\vh}{\vec{h}}

\providecommand{\vq}{\vec{q}} \providecommand{\vs}{\vec{s}}
 \providecommand{\vr}{\vec{r}}
\providecommand{\vg}{\vec{g}}
\providecommand{\vu}{\vec{u}} \providecommand{\vw}{\vec{w}}
 
\providecommand{\vz}{\vec{z}} 
 
 \providecommand{\vv}{\vec{v}}

\providecommand{\vlambda}{\vec{\lambda}}
\providecommand{\veta}{\vec{\eta}}

\usepackage{lipsum}    
\usepackage{ifthen}
\usepackage{tcolorbox}
\newboolean{showcomments}
\setboolean{showcomments}{true}
\newcommand{\oussama}[1]{\ifthenelse{\boolean{showcomments}}
{ \textcolor{red}{(Oussama says:  #1)}}{}}
\newcommand{\christos}[1]{\ifthenelse{\boolean{showcomments}}
{ \textcolor{blue}{(Christos says: #1)} } {} }
\newcommand{\yue}[1]{\ifthenelse{\boolean{showcomments}}
{ \textcolor{magenta}{(Yue says:  #1)}}{}}

\providecommand{\abs}[1]{\lvert#1\rvert}
\providecommand{\norm}[1]{\lVert#1\rVert}

\renewcommand{\vec}[1]{\ensuremath{\boldsymbol{#1}}}
\providecommand{\mat}[1]{\ensuremath{\boldsymbol{#1}}}

\providecommand{\calM}{\mathcal{M}}

 \providecommand{\mB}{\mat{B}}
\providecommand{\mC}{\mat{C}}

\providecommand{\mI}{\mat{I}}  
\providecommand{\mB}{\mat{B}}  
 \providecommand{\mP}{\mat{P}}

\providecommand{\mGm}{\mat{\Gamma}} \providecommand{\mG}{\mat{G}}

\usepackage{balance}
\providecommand{\va}{\vec{a}} \providecommand{\vb}{\vec{b}}
 
\providecommand{\vh}{\vec{h}}

\providecommand{\vq}{\vec{q}} \providecommand{\vs}{\vec{s}}
 \providecommand{\vr}{\vec{r}}
 
\providecommand{\vg}{\vec{g}}
\providecommand{\vu}{\vec{u}} \providecommand{\vw}{\vec{w}}
 
\providecommand{\vz}{\vec{z}} 
 
 \providecommand{\vv}{\vec{v}}

\providecommand{\vlambda}{\vec{\lambda}}

\providecommand{\vbeta}{\vec{\beta}}

\newcommand{\argmin}{\operatornamewithlimits{argmin}}

\providecommand{\vxi}{\vec{\xi}}

\usepackage[utf8]{inputenc} 
\usepackage[T1]{fontenc}    
\usepackage{hyperref}       
\usepackage{url}            
\usepackage{booktabs}       
\usepackage{amsfonts}       
\usepackage{nicefrac}       
\usepackage{microtype}      
\usepackage{xcolor}         
\usepackage{cite} 
\usepackage{bbold}
\usepackage{graphicx} 
\usepackage{bmpsize}
\usepackage{amsmath,amsfonts}

\begin{document}

\title{Asymptotic Behavior of Multi--Task Learning: Implicit Regularization and Double Descent Effects}

\author{Ayed M. Alrashdi, Oussama Dhifallah, and Houssem Sifaou
\thanks{
}
\thanks{
(Corresponding author: A. M. Alrashdi.)}
\thanks{A. M. Alrashdi is with the Department of Electrical Engineering, College of Engineering, University of Ha'il, P.O. Box 2440, Ha'il, 81441, Saudi Arabia (e-mail: am.alrashdi@uoh.edu.sa).}
\thanks{O. Dhifallah was with the John A. Paulson School of Engineering and Applied Sciences, Harvard University, Cambridge, MA 02138, USA (e-mail: oussama\_dhifallah@g.harvard.edu).}
\thanks{H. Sifaou is with the Department of Electrical and Electronic Engineering, King’s College London – Strand, London WC2R 2LS – London, UK (e-mail: houssem.sifaou@kcl.ac.uk).}
}

\maketitle

\begin{abstract}
Multi--task learning seeks to improve the generalization error by leveraging the common information shared by multiple related tasks. One challenge in multi--task learning is identifying formulations capable of uncovering the common information shared between different but related tasks.
This paper provides a precise asymptotic analysis of a popular multi--task formulation associated with misspecified perceptron learning models.
The main contribution of this paper is to precisely determine the reasons behind the benefits gained from combining multiple related tasks. Specifically, we show that combining multiple tasks is asymptotically equivalent to a traditional formulation with  additional regularization terms that help improve the generalization performance. Another contribution is to empirically study the impact of combining tasks on the generalization error. In particular, we empirically show that the combination of multiple tasks postpones the double descent phenomenon and can mitigate it asymptotically.
\end{abstract}

\IEEEpeerreviewmaketitle

\section{Introduction}\label{intro}
\subsection{Motivation} 
Multi--task learning \cite{tra_lrn_fr,smulti,tra_lrn} is a promising technique for improving generalization performance. It consists of leveraging common information shared among several related tasks to enhance the generalization performance associated with each individual task. 
One of the main challenges in multi--task learning is to identify learning formulations that can benefit each separate task. 
This paper considers a popular multi--task formulation \cite{muti_ref_f} associated with misspecified perception learning models (see equation~\eqref{multi_task}).
Recent literature \cite{csvm} shows that this formulation can identify the common information that may benefit individual tasks.
Specifically, it shows that this multi--task formulation leads to superior generalization performance than traditional formulations. 
This work provides a sharp asymptotic analysis of the multi--task setup described in \cite{muti_ref_f}. 
In particular, our analysis reveals an asymptotic equivalent formulation of the multi--task problem. 
The asymptotic predictions are then used to identify an equivalent formulation. Moreover, our analysis illustrates that the considered multi--task formulation is asymptotically equivalent to a traditional formulation with additional regularization terms that are the main cause of the generalization improvements.

Classical learning theory \cite{dd_1} suggests that the generalization error exhibits a U-shaped curve pattern. That is, the generalization first decreases until it reaches a minimum. Classical thinking identifies this region as the under-fitting regime. After the minimum, the learning model may over-fit which causes a poor performance on new data samples. 
\begin{figure}[h!]
    \centering
    \subfigure[]{
        \includegraphics[width=0.48\linewidth]{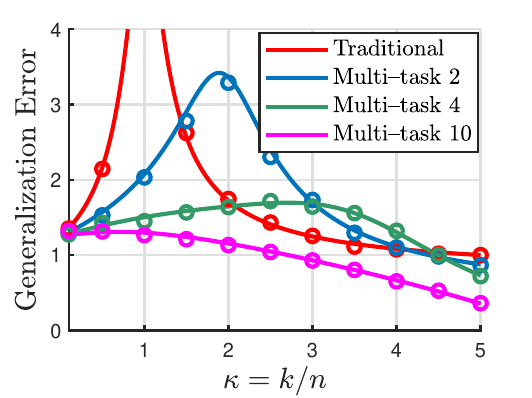}}
    \subfigure[]{
    \includegraphics[width=0.48\linewidth]{ 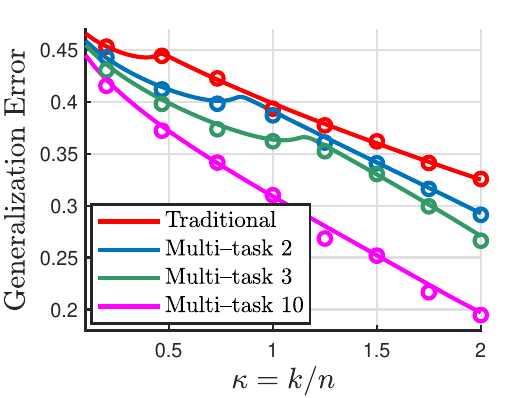}}
    \caption{Solid lines: Theoretical predictions. Circles: Numerical simulations for the multi--task formulation. {\bf (a)} A squared loss and a linear regression model. {\bf (b)} A logistic loss and a binary classification model.
    The results show a double descent pattern in the generalization error: the sweet spot is zero for the regression model and strictly positive for the classification model.
Note that the position of the interpolation threshold varies based on how many tasks are included. It is also evident that increasing the number of tasks contributes to improved generalization performance.
}
        \label{fintro2}
\end{figure}
In this regime, the generalization error is monotonically increasing as a function of the problem parameters. The objective is then to identify the location of the minimum known as the \textit{sweet spot}. Modern machine learning methods \cite{dd_1} violate this property. Instead, many machine learning methods follow what is known as the \textit{double descent curve} (see the references \cite{dd_1,BMM18,ddpap3}).
The generalization error of these models initially decreases, then increases until it hits a peak referred to as the \textit{interpolation threshold}. Beyond this peak, the generalization error declines monotonically with respect to the model parameters.
The study of such learning models has recently attracted significant attention since they violate classical results. Moreover, recent efforts \cite{dd_2,chis20bc,mu19harm} towards providing precise analysis of the double descent phenomenon focus on the single task problem. In particular, they present a theoretical understanding of the double descent phenomenon in regression and classification models as a function of $\kappa=k/n$. Here, $k$ is the number of parameters, and $n$ denotes the size of the training set. In this paper, we empirically study the impact of combining multiple related tasks on the behavior of the generalization error.
Our investigations indicate that the position of the interpolation threshold is influenced by the number of combined tasks. Moreover, the double descent effect can be reduced by aggregating a large number of tasks. Figure \ref{fintro2} illustrates these observations. It examines a regression model using the squared loss and a binary classification model using the logistic loss. In both cases, $T$ tasks are combined following the formulation in \cite{muti_ref_f}. The generalization error in both models exhibits a double descent pattern, with the interpolation threshold shifting to higher values as the number of tasks increases. Additionally, the results in Figure \ref{fintro2} indicate that combining a sufficiently large number of tasks can help mitigating the interpolation threshold.
\subsection{Summary of Contributions}
The main goal of this work is to provide a precise investigation of the effects of learning different yet related tasks, following the formulation presented in \cite{muti_ref_f}. Our analysis shows that the combination of multiple related tasks is asymptotically equivalent to a traditional formulation with an additional regularization term. The regularization is given in an explicit form. Moreover, we show that the additional regularization depends on the similarity between the tasks and helps improve the generalization performance. Our analysis starts by providing a sharp asymptotic analysis of the popular multi--task formulation introduced in \cite{muti_ref_f} for a fixed number of tasks $T$. Specifically, we show that the generalization error associated with the multi--task learning formulation concentrates in the large system limit for fixed $T$.
By solving a low--dimensional deterministic optimization problem, the asymptotic limit can be explicitly determined.
The resulting asymptotic predictions are subsequently employed to analyze the performance of the multi--task formulation in the regime where the number of combined tasks tends to infinity after the problem dimensions have been taken to infinity at a fixed rate.
The given analysis employs a Gaussian equivalence theorem known as the convex Gaussian min--max theorem (CGMT) \cite{chris:151}. 

Our precise characterization is valid for general convex loss functions, particularly a squared loss is used for regression tasks, and a logistic loss is employed for binary classification tasks.
Furthermore, it is valid for a broad class of generative models.
We specialize our general theoretical results to widely used regression and classification models. Empirical investigations show that the studied model exhibits a double descent phenomenon.
In particular, they demonstrate that the generalization error associated with the standard formulation is strictly decreasing after reaching the interpolation threshold at a value $\kappa^\star$. Additionally, they show that the combination of $T$ related tasks shifts the interpolation threshold by a factor that depends on $T$.
\subsection{Related Work}
The concept of multi--task learning \cite{tra_lrn_fr,smulti} is associated with various machine learning approaches, including transfer learning \cite{tr_learn1,tra_lrn}. These methods are similar in that they leverage information from different yet related tasks to enhance generalization error. The key difference lies in their objective: multi--task learning aims to improve the generalization performance across all learning models, whereas transfer learning focuses on using information from previously solved tasks to enhance the generalization error of a specific target task.
Recent efforts \cite{arg_07,tra_lrn,xue07a} consider modeling the relatedness between the tasks. For instance, the work in \cite{tra_lrn} models the relatedness between the tasks in terms of the correlation between the shared parameters. Another approach \cite{xue07a} models the relatedness in terms of the prior distribution of the shared parameters. A different line of work \cite{tra_lrn} focuses on providing formulations that can uncover the shared information between the tasks. This work precisely analyzes a generalized version of the popular multi--task formulation introduced in \cite{tra_lrn}. The approach in \cite{tra_lrn} provides a natural extension of the support vector machine to a multiple task setting. We extend this formulation to solve general linear regression and binary classification learning problems. 

While most research works focus on the practical aspects of the multi--task setting \cite{p1_mtask,p2_mtask,p3_mtask}, there have been several studies \cite{cter_svm,t1_mtask,t2_mtask,tiomoko20} that focus on providing precise performance analysis. Our work is particularly related to the analysis in \cite{tiomoko20}, which considers the least square support vector machine formulation. Compared to \cite{tiomoko20}, our contribution differs as follows. The analysis presented in this work is more general as it is valid for general convex formulations. In addition, this paper provides a precise characterization of the regularization effects of the multi--task formulation in \cite{tra_lrn} and examines the impact of task combination on the double descent phenomenon.  

The analysis presented in this paper is aligned with recent literature on the precise high--dimensional analysis of convex regression formulations \cite{chris:151,ouss19,dhif_trf,9698222, sifaou2021precise,alrashdi2019precise} and convex classification formulations \cite{ea19bc,dhifallah2020,chis20bc}. A common tool used in this research direction is the convex Gaussian min--max theorem (CGMT) \cite{gordon,stojnic13,chris:151, akhtiamov2025novel}. In this paper, we analyze the multi--task formulation in \cite{tra_lrn} associated with misspecified perceptron learning models using an extended version of the CGMT. A closely related work is the analysis presented in \cite{chris:151}, which uses the CGMT framework to precisely analyze a general convex regression formulation with possible inseparable loss function and regularization. Compared to \cite{chris:151}, our contribution differs as follows. The analysis presented in \cite{chris:151} assumes that the input feature vectors form a Gaussian--distributed matrix with independent and identically distributed (i.i.d) components. In this paper, the input vectors from different tasks comprise a block diagonal matrix where the diagonal blocks are Gaussian with i.i.d elements, and the off-diagonal blocks are zero. In this case, the analysis presented in \cite{chris:151} is not applicable. We essentially need an extended version of the CGMT that is called the multivariate CGMT \cite{dhi21inherent, akhtiamov2025novel}.
\subsection{ Notations} 
In our notation, column vectors are expressed using bold lower-case letters (e.g., $\mathbf a$), while matrices are represented by bold upper-case letters (e.g., $\mathbf A$).
The $i^{th}$ entry of a vector $\mathbf a$ is denoted by $a_i$, while its $\ell_2$--norm is denoted by $\|\mathbf{a}\|_2$.
The symbols $\mathbf{0}_p$ and $\mathbf{I}_p$ indicate the all-zeros vector of size $p$ and the $p \times p$ identity matrix, respectively.
The notations $(\cdot)^{\top}$ and $(\cdot)^{-1}$ represent the vector/matrix transpose and inverse operators, respectively. 
The statistical expectation is represented by $\mathbb{E}[\cdot]$, while the probability is indicated by $\mathbb{P}(\cdot)$.
The notation $\circ$ is used to designate the Hadamard product, i.e., $(\mathbf{A} \circ \mathbf{B})_{ij} = \mathbf{A}_{ij} \mathbf{B}_{ij}$, where $\mathbf{A}_{ij}$ is the $(i,j)$-th element of $\mathbf{A}$.
We write $``\xrightarrow{p \to \infty} "$ to indicate convergence in probability as $p \to \infty$.
The letters $G_1$ and $G_2$ are reserved to represent two independent standard Gaussian random variables. 
Finally, the function $\calM_{\ell(y;.)}(\cdot;\cdot)$ is used to denote the Moreau envelope function associated with the loss function $\ell(y;.)$, and it is defined as (with parameter $b > 0$)
\begin{align}\label{m_env}
\calM_{\ell(y;.)}(a;b)=\min_{x \in \mathbb{R}}\ell(y;x)+\frac{1}{2b}(x-a)^2.
\end{align}
\section{Learning Models}\label{model}
\subsection{Training Model}
We consider a scenario in which the learner has access to $T$ distinct learning tasks. For the $t^{th}$ task, the training dataset is given by $\lbrace (\va_{t,i}, y_{t,i}) \rbrace_{1 \leq i \leq n_t}$, where $\va_{t,i} \in \mathbb{R}^p$ represents the feature vector and $y_{t,i}$ is the corresponding label ($ \forall i \in \lbrace 1, \dots, n_t \rbrace,~ t \in \lbrace 1, \dots, T \rbrace$). 
In this work, we assume the labels are generated based on the following model:
 \begin{align}\label{mt_model}
  y_{t,i} = \varphi(\va_{t,i}^\top \vxi_t),
  \end{align}
where $\vxi_t\in\mathbb{R}^p$ is a hidden vector associated with the $t^{th}$ task, and $\varphi(\cdot)$ is a function that may be deterministic or probabilistic. Furthermore, we assume that the learning tasks are related in the following manner:
\begin{align}\label{sim_md}
\vxi_t = \sigma \vv_t  + \vv_0,~\forall t \in \lbrace 1,\dots,T \rbrace,
\end{align}
where $\vv_t\in\mathbb{R}^p$ is a task specific vector and $\vv_0\in\mathbb{R}^p$ is a shared vector between all the tasks. 
Observe that the parameter $\sigma \in \mathbb{R}$ governs the degree of similarity among the tasks. Based on this, we define the similarity between tasks using the quantity $\rho$, given by
\[
\rho = \frac{1}{1 + \sigma^2}.
\]
Note that $\rho \in [0, 1]$, where values of $\rho$ approaching $1$ indicate that the tasks are highly similar, while lower values suggest that the tasks are dissimilar. Hereafter, we refer to $\rho$ as the \textit{similarity measure}.

%

In this work, we consider a misspecified learning scenario in which the learner has access only to partial observations of the input vectors during training process. Specifically, for each input vector $\va_{t,i}$, only a subset of its components, denoted by $(a_{t,ij},~j \in \mathcal{S})$, is available to the learner, where $\mathcal{S} \subset \lbrace 1, \dots, p \rbrace$. 
To simplify the analysis, we assume that the subset $\mathcal{S}$ is fixed and does not depend on the sample index $i \in \lbrace 1, \dots, n_t \rbrace$ or the task index $t \in \lbrace 1, \dots, T \rbrace$. Furthermore, we assume that the cardinality of $\mathcal{S}$ is fixed at $k$, with $1 \leq k \leq p$.

The analysis presented in this work is valid for input vectors and hidden vectors generated randomly as summarized in the subsequent assumption.
\begin{assumption}[Input/Hidden Vectors]
\label{rad_fv}
For any $t\in\lbrace 1,\dots,T \rbrace$, the input vectors $\lbrace \va_{t,i} \rbrace_{1 \leq i \leq n_t}$ are assumed to be known and drawn independently from a standard Gaussian distribution. 
The vectors $\vv_t\in\mathbb{R}^p$ and $\vv_0\in\mathbb{R}^p$ are assumed to be independent of the input vectors and are generated independently from a uniform distribution on the unit sphere.
Without loss of generality, we assume that both $\vv_t$ and $\vv_0$ are unit-norm vectors.
Furthermore, the set $\mathcal{S}$ is assumed to be selected uniformly at random.
\end{assumption}
In addition, the results hold in the high--dimensional asymptotic regime, where the problem dimensions $p$, $k$, and $n_t$
grow large and satisfy the next assumption.
\begin{assumption}[High--dimensional Asymptotics]
\label{highdim}
For any $t\in\lbrace 1,\dots,T \rbrace$, we assume that the number of samples and the number of known components of the input vector satisfy $n_t=n_t(p)$ and $k=k(p)$ with $\alpha_{t,p}=p/n_t(p) \to \alpha_t>0$ and $\kappa_{t,p}=k(p)/n_t(p) \to \kappa_t>0$ as $p\to \infty$, where $\kappa_t \leq \alpha_t$. Furthermore, the number of tasks $T\geq 1$ is independent of the dimension $p$. 
\end{assumption}
This paper relies on specific assumptions regarding the distribution of the input vectors, the generative model in \eqref{mt_model}, and the distribution of the hidden vectors. We emphasize that these assumptions are crucial for the validity of our asymptotic analysis. An interesting direction for future research is to relax the Gaussianity assumption by demonstrating universality results (see, e.g., \cite{univ_1,univ_2}).
\subsection{A Multi--Task Learning Algorithm}
Given the similarity among the learning tasks, a widely used training strategy \cite{muti_ref_f} involves jointly learning the collection of hidden vectors $\lbrace \vxi_t \rbrace_{1 \leq t \leq T}$. This approach incorporates a regularization term that reflects the task similarity structure given in \eqref{sim_md}. In particular, a commonly adopted formulation in multi--task learning takes the following general form:
\begin{align}\label{multi_task}
\lbrace \widehat{\vw}_t \rbrace_{1\leq t \leq T}&=\argmin_{ \lbrace \vw_t\rbrace_{1\leq t \leq T} } \sum_{t=1}^{T} \frac{1}{n_t} \sum_{i=1}^{n_t} \ell \Big(y_{t,i};\vb_{t,i}^\top \vw_t \Big)  \nonumber\\&+\frac{\gamma_1}{2} \sum_{t=1}^{T} \norm{\vw_t}^2+\frac{\gamma_2}{2} \sum_{t=1}^{T} \norm{ \vw_t-\bar{\vw} }^2.
\end{align}
In the formulation above, the vector $\vb_{t,i} \in \mathbb{R}^k$ is obtained by concatenating the components of the input vector $\va_{t,i}$ corresponding to the index set $\mathcal{S}$. 
The term $\bar{\vw}$ represents the mean of the optimization vectors $\lbrace \vw_t \rbrace_{1 \leq t \leq T}$, defined as $\bar{\vw} = \sum_{t=1}^{T} \vw_t / T$. 
The parameter $\gamma_1 \geq 0$ regulates the strength of the regularization applied to each task, while $\gamma_2 \geq 0$ governs the regularization imposed on the average model. The formulation in \eqref{multi_task} is applicable to general loss functions used in both regression and classification tasks.
The loss function $\ell(y; x)$ can be expressed in one of the following two general forms:
\begin{align}\label{regclass_loss}
\ell(y; x) = f(y - x), \quad \text{or} \quad \ell(y; x) = f(yx),
\end{align}
where the first expression corresponds to regression tasks, and the second is applicable to classification tasks. The function $f(\cdot)$ denotes a general scalar function. Note that the optimization problem in \eqref{multi_task} reduces to a standard (per-task) learning problem when $\gamma_2 = 0$. Furthermore, the formulation is symmetric across tasks when all tasks have an equal number of training samples. Throughout this work, we refer to the case $\gamma_2 = 0$ as the \textit{traditional formulation}, and to the case $\gamma_2 > 0$ as the \textit{multi--task formulation}.
\subsection{Performance Measure}
The main goal of this work is to precisely analyze the performance of the multi--task learning approach on unobserved test data. We use the \textit{generalization error} to measure the performance of the considered multi--task learning formulation.
To reach a formal definition of the generalization error, we start by defining the vector $\widehat{\vbeta}_t\in\mathbb{R}^p$ as follows
\begin{align}\label{beta_hat}
\widehat{\vbeta}_t(\mathcal{S})=\widehat{\vw}_t,~\text{and}~\widehat{\vbeta}_t(\mathcal{S}^c)=\vec{0}_{p-k},
\end{align}
where $\widehat{\vbeta}_t(\mathcal{S})$ denotes the components of $\widehat{\vbeta}_t$ with index in the set $\mathcal{S}$, whereas $\vec{0}_{p-k}$ corresponds to the all-zero vector of size $p-k$.
In this paper, it is assumed that the $t^{th}$ task predicts the label of any new test sample $\va_{t,\text{new}}\in\mathbb{R}^{p}$ as follows
\begin{align}
\widehat{y}_{t,\text{new}}=\widehat{\varphi} [\widehat{\vbeta}_t^\top \va_{t,\text{new}}].
\end{align}
In the above equation, $\widehat{\varphi}(\cdot)$ denotes a pre-defined scalar function. Now, we are ready to define the generalization error. In particular, the \textbf{generalization error} associated with the $t^{th}$ task can be defined as follows
\begin{align}\label{gener_trg}
\mathcal{E}_{p,t,\text{test}}=\frac{1}{4^\vartheta} \mathbb{E}\Big[ \Big( \varphi(\vxi_t^\top \va_{t,\text{new}})-\widehat{\varphi}(\widehat{\vbeta}_t^\top  \va_{t,\text{new}}) \Big)^2 \Big].
\end{align}
Here, the parameter $\vartheta$ is set to $\vartheta = 0$ for regression tasks and $\vartheta = 1$ for binary classification tasks.
Note that the expectation in \eqref{gener_trg} is taken with respect to the distribution of $\va_{t,\text{new}}$ and $\varphi(\cdot)$. 

\vspace{1mm}
\noindent{\bf Validation Models}: In this paper, we present a precise asymptotic analysis of the formulation in \eqref{multi_task}. Our theoretical derivations provided in the appendix are applicable to a broad class of convex loss functions, and to general models satisfying \eqref{mt_model}. 
However, for clarity and to facilitate interpretation of the results, we focus on two widely used loss functions in this paper: the \textit{squared loss} and the \textit{logistic loss}, defined as follows:
\begin{align}
\ell(y;x)=\frac{1}{2}(x-y)^2,~\text{and}~\ell(y;x)=\log( 1+e^{-xy} ),
\end{align}
respectively. 
These loss functions are employed to learn regression and classification models, respectively. In the case of the \emph{regression} model, we assume that both $\varphi(\cdot)$ and $\widehat{\varphi}(\cdot)$ correspond to the \emph{identity} function. For the \emph{classification} model, both functions are taken to be the \emph{sign} function. Throughout the paper, we refer to these setups as the \textit{linear regression model} and the \textit{binary classification model}, respectively.
\section{Symmetric Multi--Task Formulation}\label{lntasks}
In this section, we present a precise high--dimensional analysis of the multi--task learning approach. For the simplicity of analysis, we consider the case when all the tasks have the same training set size, that is, $n_t=n, \forall t\in\lbrace 1,\dots,T \rbrace$. 
The general case is presented in Section~\ref{gntasks}.
\subsection{Precise Asymptotic Predictions}
The asymptotic predictions of the symmetric multi--task learning require a few definitions. We start by defining the following low--dimensional deterministic formulation
\begin{align}\label{det_form_asy}
&\min_{q,r\geq 0}\max_{\eta>0}~\frac{1}{2}(\gamma_1-\eta)(q^2+r^2)+\frac{q^2}{2} \  \frac{\gamma_2+\eta}{1+(1-\rho)\frac{\gamma_2}{\eta T}} \  \mathcal{G}(T,\eta) 
\nonumber \\ &+\mathbb{E} \left[ \mathcal{M}_{\ell(Y;.)}\left(r H+q S;\frac{\kappa}{(\gamma_2+\eta)} \Big(1+\frac{\gamma_2}{\eta T} \Big) \right) \right],
\end{align}
where the function $\mathcal{G}(\cdot,\cdot)$ and the random variable $Y$ are defined as follows
\begin{align}
&\mathcal{G}(T,\eta)=1-\frac{\gamma_2 \rho T}{\eta T+\gamma_2 (1-\rho+\rho T)}, \nonumber \\ &Y=\varphi \bigg( \frac{1}{\sqrt{\rho}} \left[S \sqrt{\frac{\kappa}{\alpha}} +Z\sqrt{1-\frac{\kappa}{\alpha}} \  \right] \bigg),
\end{align}
and $Z, H$ and $S$ are standard Gaussian independent random variables.
The expectation in the objective function of the problem in \eqref{det_form_asy} is taken over the randomness of $H$, $S$ and $Y$. 

Now, we are in a position to state our first theoretical prediction summarized in the next theorem.
\begin{theorem}[Symmetric Multi--Task Analysis]\label{lem_same}
Let Assumptions \ref{rad_fv}--\ref{highdim} hold. In addition, assume that all tasks have the same training set size, i.e., $\alpha_t = \alpha$ for all $t \in \lbrace 1, \dots, T \rbrace$. Under these conditions, the generalization error defined in \eqref{gener_trg} associated with the $t^{\text{th}}$ task converges in probability to the following limit:
\begin{align}\label{gen_conv}
{\mathcal{E}}_{p,t,\text{test}} \xrightarrow{p\to\infty} \frac{1}{4^\vartheta} \mathbb{E}\left[ \left( \varphi(c_0 G_1) -\widehat{\varphi}(c_{1,T} G_1+c_{2,T} G_2) \right)^2 \right].
\end{align}
In the above, $G_1$ and $G_2$ are independent standard Gaussian random variables. Also, $c_0$, $c_{1,T}$ and $c_{2,T}$ are constants defined as follows
\begin{align}
&c_0=\frac{1}{\sqrt{\rho}},~c_{1,T}=q_T^\star \sqrt{\frac{\kappa}{\alpha}} ,~ \text{and} \nonumber \\  &c_{2,T}=\sqrt{ \left(1-\frac{\kappa}{\alpha} \right) (q_T^\star)^2+(r_T^\star)^2 },
\end{align}
where $r_T^\star$ and $q_T^\star$ are the optimal solutions of the scalar optimization problem in \eqref{det_form_asy}. 
\end{theorem}
\begin{proof}
The result of Theorem~\ref{lem_same} is a special case of Theorem~\ref{th_gmtask}. Please refer to the appendix for more details.
\end{proof}
The technical analysis shows that the deterministic formulation in \eqref{det_form_asy} is strictly convex in the minimization variables. This implies the uniqueness of its optimal solution. Note that the asymptotic predictions in Theorem \ref{lem_same} show that the multi--task formulation in \eqref{multi_task} can be fully characterized after solving a three-dimensional deterministic formulation. Interestingly, the results in Theorem \ref{lem_same} reduces the complexity of \eqref{multi_task}, which depends on $T$, to a three-dimensional optimization problem. This allows the analysis of the multi--task formulation in \eqref{multi_task} when the number of tasks grows to infinity.

Now, we provide another simulation example to verify the results stated in Theorem \ref{lem_same}. Figure \ref{finsim1} considers the linear regression and binary classification models. The asymptotic predictions stated in Theorem \ref{lem_same} can be validated by observing that they are in excellent agreement with the actual performance of the multi--task formulation.
\begin{figure}[h!]
    \centering
    \subfigure[]{\label{finsim1a}
        \includegraphics[width=0.48\linewidth]{ 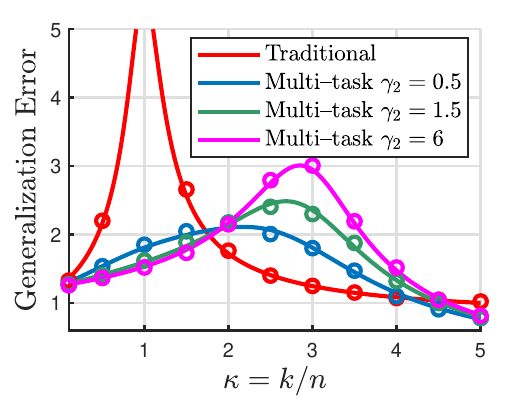}}
    \subfigure[]{\label{finsim1b}
    \includegraphics[width=0.48\linewidth]{ 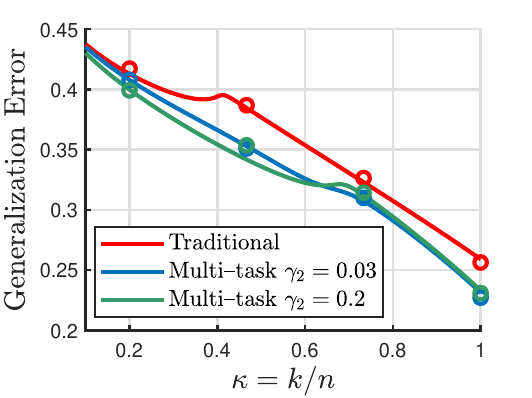}}
     
    \caption{Continuous lines: Theoretical predictions. Circles: Numerical simulations for the multi--task formulation. {\bf (a)} We consider the linear regression model and the squared loss. We set $p=2000$, $\alpha=5$, $\rho=0.8$, $\gamma_1=10^{-2}$ and $T=3$. {\bf (b)} We consider the binary classification model and the logistic loss. We set $p=600$, $\alpha=1$, $\rho=0.8$, $\gamma_1=10^{-4}$ and $T=2$. The results are averaged over $25$ independent Monte Carlo trials.}
        \label{finsim1}
\end{figure}
Figure \ref{finsim1a} considers the multi--task formulation for $T=3$. Also, it investigates the impact of the regularization strength $\gamma_2$ on the double descent phenomenon for the linear regression model. We can see that the location of the interpolation threshold depends on the regularization strength $\gamma_2$. That is, the location of the peak increases we increase the value of $\gamma_2$. This suggests that the location of the interpolation threshold moves from $1$ to $T$ smoothly in terms of $\gamma_2$. We can also see that the generalization error in the interpolation threshold first deceases and then it increases as we increase $\gamma_2$. Figure \ref{finsim1b} also illustrates the dependence of the interpolation threshold on the value of $\gamma_2$ for the binary classification model. It suggests that the double descent for the traditional formulation, occurring at $\kappa^\star\approx 0.41$, is mitigated for small values of $\gamma_2$. Then, it appears again at $T \kappa^\star$ as we increase $\gamma_2$. Finally, Figure \ref{finsim1} recommends that a small value of $\gamma_2$ is capable of reducing the double descent effects for the binary classification model employing a logistic loss.

\subsection{Combining Large Number of Tasks}
Here, we study the properties of the multi--task formulation in the regime where the number of tasks $T$ grows to infinity after the dimensions $p$, $n$ and $k$ have been taken to infinity.
We start our analysis by defining the following scalar formulation
\begin{align}\label{det_form_asy2}
\min_{q,r\geq 0}\max_{\eta>0}&~\frac{1}{2}(\gamma_1-\eta)(q^2+r^2)+\frac{q^2}{2} (\gamma_2+\eta)\Big(1-\frac{\gamma_2 \rho}{\eta+\gamma_2 \rho} \Big)
\nonumber \\ &+\mathbb{E} \Big[ \mathcal{M}_{\ell(Y;.)}\Big(r H+q S;\frac{\kappa}{\gamma_2+\eta} \Big) \Big],
\end{align}
where $H$ and $S$ are independent standard Gaussian random variables.
The theoretical result is stated in the next Lemma.
\begin{lemma}[Large Number of Tasks]\label{lem_asy}
Suppose that the assumptions \ref{rad_fv}-\ref{highdim} are satisfied. Moreover, assume that the tasks have the same training set size, i.e., $\alpha_t=\alpha, \forall t\in\lbrace 1,\dots,T \rbrace$. Then, for any $\zeta > 0$, the generalization error corresponding to the $t^{th}$ task converges in probability as follows 
\begin{align}
\lim_{T\to+\infty}\lim_{p\to+\infty}\mathbb{P}\big( \abs{ {\mathcal{E}}_{p,t,\text{test}} -  {\mathcal{E}}_{t,\text{test}} } < \zeta \big)=1.\nonumber
\end{align}
The scalar ${\mathcal{E}}_{t,\text{test}}$ is defined as follows
\begin{align}\label{gen_conv_tasy}
{\mathcal{E}}_{t,\text{test}}= \frac{1}{4^\vartheta} \mathbb{E}\left[ \left( \varphi(c_0 G_1) -\widehat{\varphi}(c_1 G_1+c_2 G_2) \right)^2 \right],
\end{align}
where $c_1$ and $c_2$ are constants defined as follows
\begin{align}
\label{c0c1}
c_1=q^\star \sqrt{\frac{\kappa}{\alpha}} ,~ \text{and} ~ c_2=\sqrt{ \left(1-\frac{\kappa}{\alpha} \right) (q^\star)^2+(r^\star)^2 },
\end{align}
with $r^\star$ and $q^\star$ being the optimal solutions of the scalar formulation given in \eqref{det_form_asy2}. 
\end{lemma}
\begin{proof}
Lemma \ref{lem_asy} follows directly from the results stated in Theorem \ref{lem_same} by letting $T \to \infty$.
\end{proof}
Lemma \ref{lem_asy} shows that the asymptotic properties of the multi--task formulation in \eqref{multi_task} can be fully characterized by solving a simple scalar problem, with the limit $T\to \infty$ taken after the limits $p,n,k\to \infty$.
Our analysis shows that the formulation in \eqref{det_form_asy2} is strictly convex in the minimization variables, which implies the uniqueness of its optimal solutions.

In the following simulation example, we validate the results of Lemma \ref{lem_asy}. In particular, we investigate the convergence behavior of the multi--task formulation under the sequential asymptotic regime in which the dimensions $p$, $k$, and $n$ tend to infinity, followed by the limit $T\to\infty$.
In Figure \ref{figsim3}, we consider the linear regression and binary classification models combined with the squared loss function.
\begin{figure}[h!]
    \centering
    \subfigure[]{
       \includegraphics[width=0.48\linewidth]{ 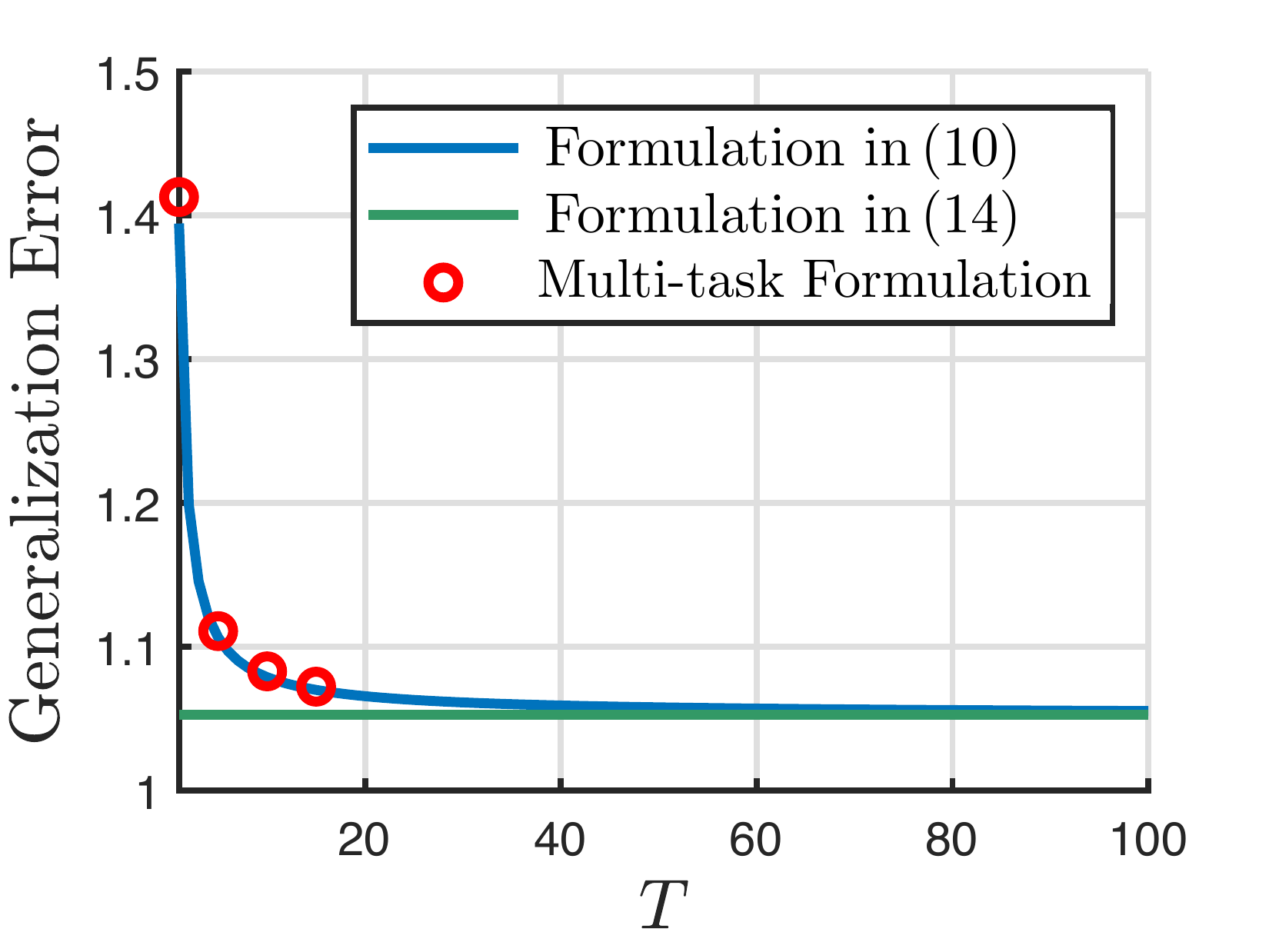}}
    \subfigure[]{
    \includegraphics[width=0.48\linewidth]{ 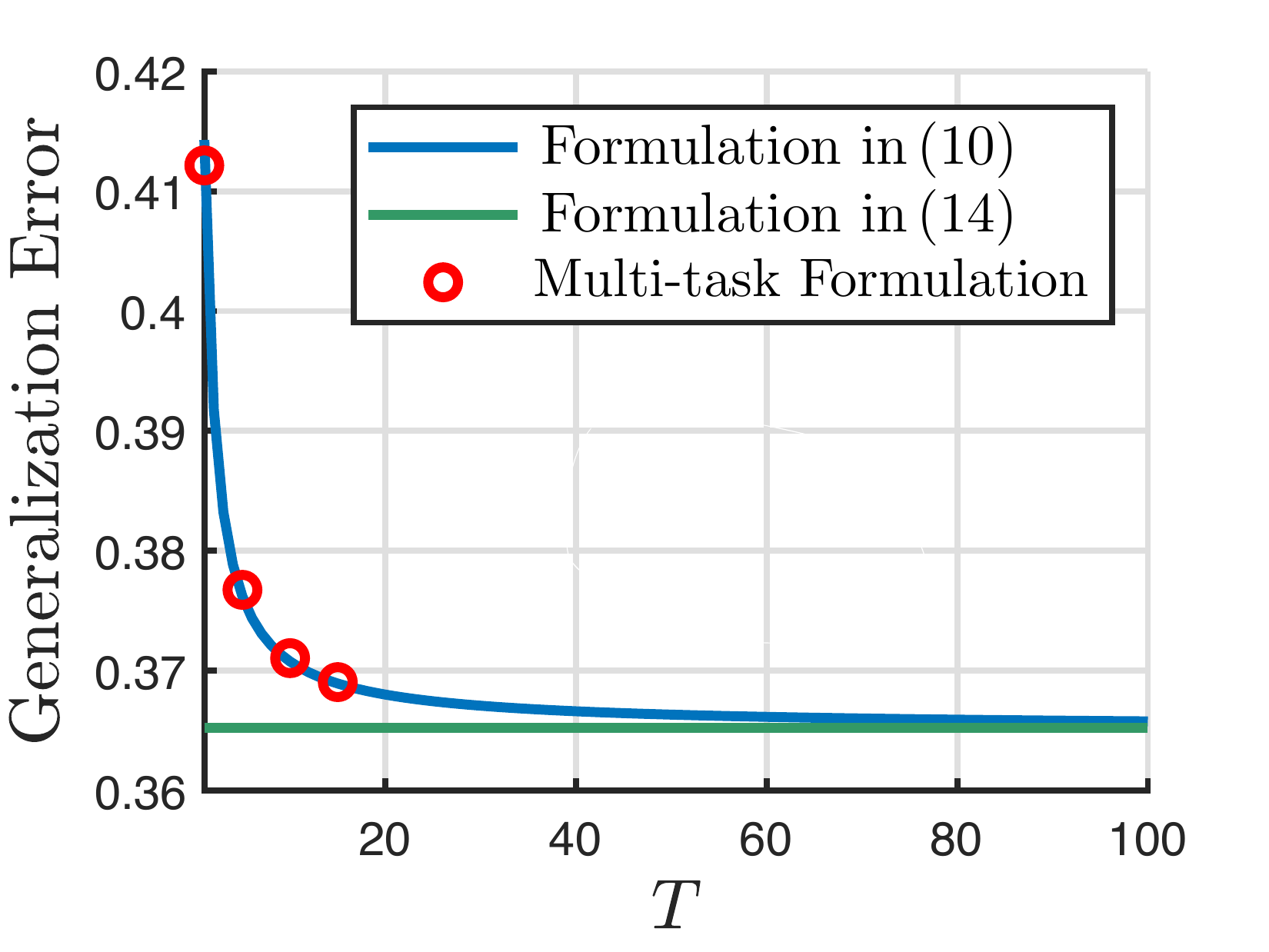}}
    \caption{Continuous line: Theoretical predictions. Circle: Numerical simulations for the multi--task formulation. {\bf (a)} We consider the linear regression model and the squared loss. We set $p=1000$, $\alpha=2$, $\kappa=0.5$, $\gamma_1=0.1$, $\gamma_2=0.5$ and $\rho=0.85$. {\bf (b)} We consider the binary classification model and the squared loss. We set $p=1000$, $\alpha=2$, $\kappa=1$, $\gamma_1=0.05$, $\gamma_2=0.2$ and $\rho=0.75$. The results are averaged over $100$ independent Monte Carlo trials.}
        \label{figsim3}
\end{figure}
In addition, we consider the generalization error of \eqref{multi_task} for values of $T$ smaller than $20$ for computational complexity reasons. First, we can see that the results in Lemma \ref{lem_asy} are in excellent agreement with the actual performance of \eqref{multi_task}. 
Note that the generalization error of \eqref{multi_task} converges to the generalization error of the deterministic formulation in \eqref{det_form_asy2}. We can also see that the limit is already achieved using a reasonable number of tasks, i.e., $T \approx 80$. Moreover, observe that  the generalization error of the multi--task formulation is strictly decreasing as a function of the number of tasks $T$. This suggests that it is always beneficial to combine more related tasks.
\subsection{Regularization Effects} 
%
%
%
%
The deterministic formulation in \eqref{det_form_asy2} is independent of the number of tasks. Essentially, the result in Lemma \ref{lem_asy} states that the generalization error of \eqref{multi_task} associated with each task can be asymptotically determined by solving the deterministic formulation in \eqref{det_form_asy2} separately at each task. Therefore, one can use this asymptotic result to determine task specific formulations that will globally lead to the same performance of the multi--task formulation. The main objective of this section is to identity $T$ formulations that can be solved separately at each task and they globally lead to the same performance of the multi--task formulation.

Before stating our main results, let us define the following formulation, which we refer to as the \textit{separate formulation}. Specifically, solving the multi--task formulation will be equivalent to solving $T$ problems separately with an additional regularization term. These problems can be written as
\begin{align}\label{smulti_task}
\widehat{\vw}_t&=\argmin_{\vw_t\in\mathbb{R}^k}~ \frac{1}{n} \sum_{i=1}^{n} \ell \Big(y_{t,i};\vb_{t,i}^\top \vw_t \Big)+\frac{\gamma_1+\gamma_2}{2} \norm{\vw_t}^2 \nonumber \\ &-\frac{\gamma_2 R(\rho)}{2} (\bar{\vxi}_{ts}^\top \vw_t)^2, \ \ {\rm for} \ \  t\in\lbrace 1,\dots,T \rbrace,
\end{align}
where $\bar{\vxi}_{ts}$ is the normalized version of the vector ${\vxi}_{ts}$. Additionally, the scalar $R(\rho)$ depends on the similarity measure $\rho$ and satisfies a fixed point equation. Particularly, the value of $R(\rho)$ is selected such that the following equality is satisfied
\begin{align}\label{fixed_equ}
&\mathbb{E}\left[ \left( \varphi(c_0 G_1) -\widehat{\varphi}(c_1 G_1 + c_2 G_2) \right)^2 \right]-\nonumber \\ &\mathbb{E}\left[ \left( \varphi(c_0 G_1) -\widehat{\varphi}(c_{1,R} G_1 + c_{2,R} G_2) \right)^2 \right]=0.
\end{align}
Here, $c_1$ and $c_2$ are the constants defined in \eqref{c0c1}. Moreover, the terms $c_{1,R}$ and $c_{2,R}$ depend on the value of $R(\rho)$ as follows 
\begin{align}
\label{c1r}
c_{1,{R}}=q_R^\star \sqrt{\frac{\kappa}{\alpha}} ,~c_{2,R}=\sqrt{ \left(1-\frac{\kappa}{\alpha} \right) (q_R^\star)^2+(r_R^\star)^2 },
\end{align}
where $q_R^\star$ and $r_R^\star$ are optimal solutions of the following deterministic formulation 
\begin{align}\label{smulti_task_asy}
\min_{q,r\geq 0}\max_{\eta>0}&~\frac{r^2}{2}(\gamma_1-\eta)+\frac{q^2}{2} (\gamma_1+\gamma_2)-\frac{\gamma_2 R(\rho)}{2} q^2 \nonumber \\ &+\mathbb{E} \Big[ \mathcal{M}_{\ell(Y;.)}\Big(r H+q S;\frac{\kappa}{\gamma_2+\eta} \Big) \Big].
\end{align}
In the next Lemma, we provide a precise analysis of the generalization error of the separate formulation introduced in \eqref{smulti_task}.
\begin{lemma}[Separate Formulation]
\label{ta5riff}
Suppose that the assumptions \ref{rad_fv}-\ref{highdim} are satisfied. Moreover, assume that the tasks have the same training set size, i.e., $\alpha_t=\alpha, \forall t\in\lbrace 1,\dots,T \rbrace$. Then, for any $\zeta > 0$, the generalization error of \eqref{smulti_task},  $\hat {\mathcal{E}}_{p,t,\text{test}}$, converges in probability as follows 
\begin{align}
\lim_{p\to+\infty}\mathbb{P}\big( \abs{ \tilde {\mathcal{E}}_{p,t,\text{test}} -  \tilde{\mathcal{E}}_{t,\text{test}} } < \zeta \big)=1. \nonumber
\end{align}
The quantity $\tilde{\mathcal{E}}_{t,\text{test}}$ is defined as
\begin{align}\label{gen_conv_tasy}
\tilde{\mathcal{E}}_{t,\text{test}}= \frac{1}{4^\vartheta} \mathbb{E}\left[ \left( \varphi(c_0 G_1) -\widehat{\varphi}(c_{1,R} G_1 + c_{2,R} G_2) \right)^2 \right],
\end{align}
where $c_{1,R}$ and $c_{2,R}$ are as defined in \eqref{c1r}.
\end{lemma}
\begin{proof} The proof follows the same techniques as in the proof of Theorem \ref{lem_same}, and it is thus omitted for brevity.
\end{proof}
With the results of Lemma~\ref{ta5riff} at hand, we can now present the main results of this section, which is stated formally next.
\begin{corollary}[Regularization Effects]\label{cor_reg}
Under the same settings as in Lemma \ref{ta5riff} and for any $\zeta > 0$, it holds
\begin{align}
\lim_{T\to+\infty}\lim_{p\to+\infty}\mathbb{P}\big( \abs{ {\mathcal{E}}_{p,t,\text{test}} -  \tilde{\mathcal{E}}_{p,t,\text{test}} } < \zeta \big)=1, \nonumber
\end{align}
where ${\mathcal{E}}_{p,t,\text{test}}$ is the generalization error of \eqref{multi_task}, and $\tilde{\mathcal{E}}_{p,t,\text{test}}$ is the generalization error corresponding to the formulation in \eqref{smulti_task}.
\end{corollary}
Note that Corollary \ref{cor_reg} provides a precise characterization of the regularization effects of the multi--task formulation in \eqref{multi_task} in the regime where the dimensions $p$, $k$, and $n$ first tend to infinity, followed by the limit $T\to\infty$.
It shows that the performance of the multi--task formulation can be achieved by solving $T$ formulations of the form in \eqref{smulti_task} separately for each task. Note that the formulation in \eqref{smulti_task} is strongly convex and cannot be solved in practice since the vector $\vxi_t$ is unknown by the learner. However, it precisely determines the reasons behind the benefits gained from combining related tasks using \eqref{multi_task}. We can see that the combination of large number of tasks leads to an additional ridge regularization with strength $\gamma_2$. Moreover, it leads to a regularization that depends on the correlation between the optimization vector and the observed components of the hidden vector $\vxi_t$. This particular regularization is the first reason behind the generalization improvement since it favors solutions aligned with the generative model in \eqref{mt_model}.

The values of $R(\rho)$ for the extreme cases $\rho=0$ and $\rho=1$ are easy to obtain. Indeed, to ensure that (14) and (21) are equivalent, we get simply $R(0)=0$ and $R(1)=1$.
Generally, the value of $R(\rho)$ should satisfy the equation in \eqref{fixed_equ}, which guarantees that the deterministic formulations in \eqref{det_form_asy2} and \eqref{smulti_task_asy} are equivalent in terms of the asymptotic generalization error. Intuitively, one can see that the equation in \eqref{fixed_equ} has a unique solution for any $\rho \in [0,1]$. This is because the generalization error associated with the deterministic problem in \eqref{smulti_task_asy} is strictly increasing as a function of $R(\rho) \in [0,1]$. Besides, the formulation in \eqref{det_form_asy2} will always lead to a value of $q$ that is between the value obtained by solving the formulation in \eqref{smulti_task_asy} for $R(0)=0$ and $R(1)=1$. This means that there exists a unique $R(\rho)\in [0,1]$ that satisfies the equation in \eqref{fixed_equ} for any $\rho\in[0,1]$. 

Moreover, one can see that, when the tasks are fully dissimilar (i.e., $\rho=0$), the multi--task formulation is only adding an additional ridge regularization with strength $\gamma_2$ asymptotically. When the tasks are fully aligned (i.e., $\rho=1$), it can also be observed that, in addition to the ridge regularization, the multi--task formulation is also adding a regularization that favors solutions with maximum correlation with $\vxi_{ts}$. Here, ${\vxi}_{ts}$ represents the vector formed by the entries of the vector $\vxi_t$ with index in the set $\mathcal{S}$.

A simulation example is now given to illustrate the results stated in Corollary \ref{cor_reg}. We consider the binary classification model and the squared loss function. In Figure \ref{figsim4b}, we consider the generalization error of \eqref{multi_task} for values of $T$ smaller than $20$ for computational complexity reasons. Figure \ref{figsim4b} first shows that our results match the actual performance of the formulations in \eqref{multi_task} and \eqref{smulti_task}. We can also notice that the generalization error of the multi--task formulation in \eqref{multi_task} converges to the generalization error corresponding to \eqref{smulti_task}. This provides an empirical verification of the results stated in Corollary \ref{cor_reg}. We can also see that the limit is already achieved with a reasonable number of tasks, i.e., $T\approx 80$.
\begin{figure}[h!]
    \centering
    \subfigure[]{\label{figsim4b}
       \includegraphics[width=0.48\linewidth]{ 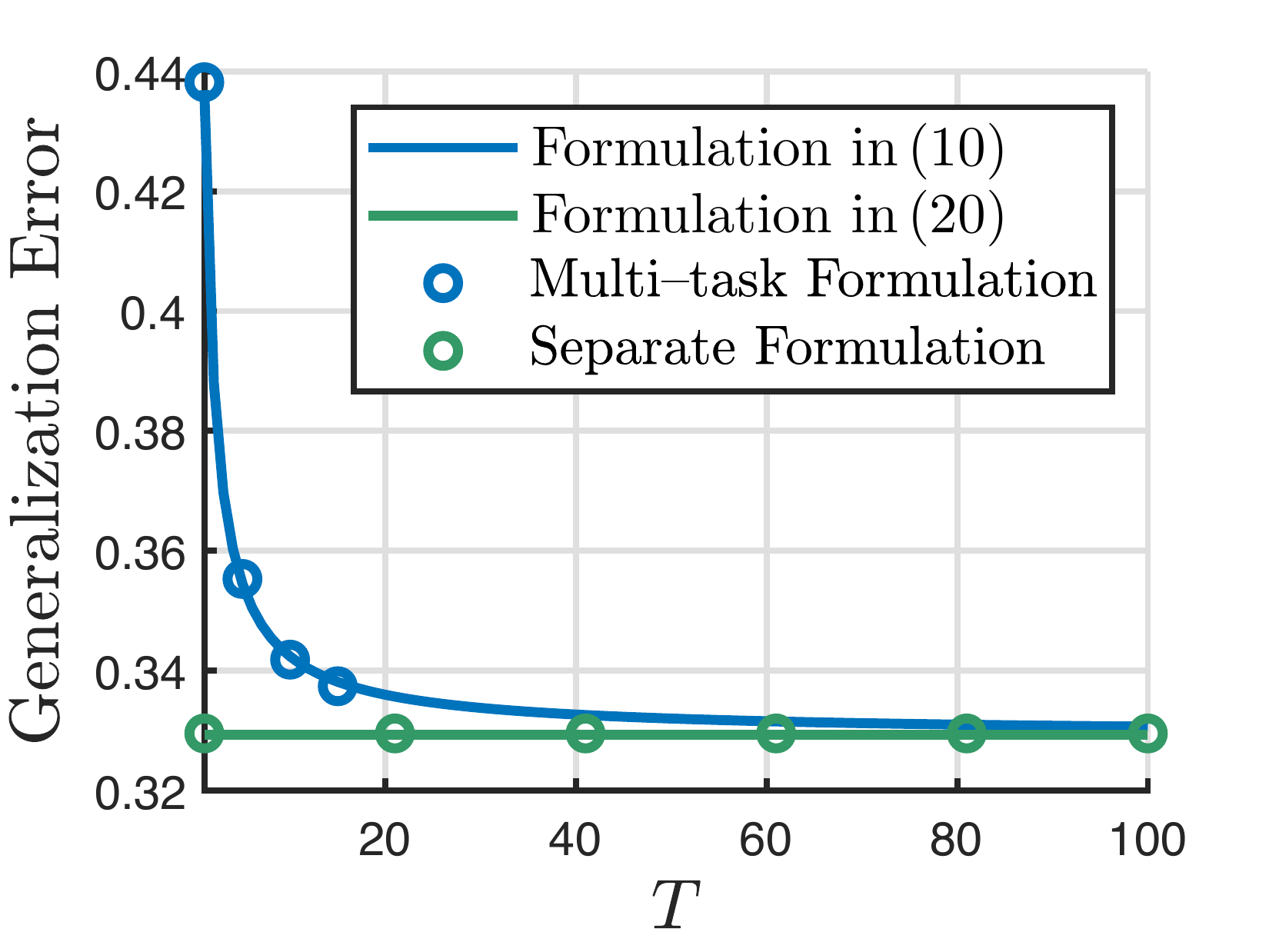}}
    \subfigure[]{\label{figsim4a}
        \includegraphics[width=0.48\linewidth]{ 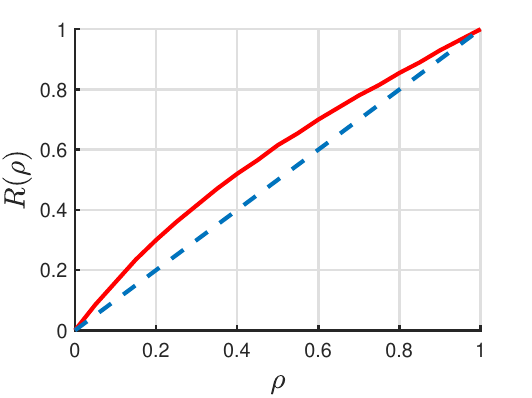}}
     
    \caption{{\bf (a)} Continuous lines: Theoretical predictions. Circles: Numerical simulations for the multi--task and separate formulations. We consider the binary classification model and the squared loss. We set $\alpha=4$, $\kappa=2$, $\gamma_1=0.1$, $\gamma_2=1$ and $\rho=0.3$. {\bf (b)} The value of $R(\rho)$ as a function of the similarity measure $\rho$. We consider the binary classification model and the squared loss. We used $p=1000$, $\alpha=2$, $\kappa=1$, $\gamma_1=0.01$, $\gamma_2=0.6$ and $\rho=0.75$. The results are averaged over $100$ independent trials.}
        \label{figsim4}
\end{figure}
Figure \ref{figsim4a} illustrates the value of $R(\rho)$ as a function of the similarity measure $\rho$. First, we can see that the value of $R(\rho)$ is always bigger than $\rho$. Furthermore, we can see that $R(\rho)$ is strictly increasing as function of the similarity measure $\rho$, where $R(1)=1$ and $R(0)=0$. This shows that the generalization error associated with the multi--task formulation improves as we increase the similarity measure.

\section{General Multi--Task Formulation}\label{gntasks}
In this section, we analyze the multi--task formulation for \emph{general} number of samples $\lbrace n_t \rbrace_{1 \leq t \leq T}$. We provide a precise characterization of the generalization error for general $\lbrace \kappa_t \rbrace_{1 \leq t \leq T}$. Before stating our theoretical predictions, we need a few definitions.
\subsection{Definitions}
We start by defining the asymptotic limit corresponding to the multi--task formulation. Specifically, define the following deterministic optimization problem
\begin{align}\label{det_form_m}
\min_{q_t,r_t \geq 0}\max_{\substack{\eta_t\\ \mC \succ 0}}&~\frac{1}{2}\sum_{t=1}^{T}(\gamma_1-\eta_t)(q_t^2+r_t^2)+\frac{1}{2}\vq^\top\mB^{-1}\vq 
\nonumber \\ &+\sum_{t=1}^T\mathbb{E} \Big[ \mathcal{M}_{\ell(Y_{t};.)}(r_t H_{t}+q_t S_{t}; \kappa_t C_{tt}^{-1}) \Big],
\end{align}
where, here, the vector $\vq\in\mathbb{R}^T$ is formed by the concatenation of the variables $\lbrace q_t \rbrace_{1 \leq t \leq T}$. Furthermore, the scalar $C_{tt}^{-1}$ denotes the $t^{th}$ diagonal element of the matrix $\mC^{-1}$, where the matrix $\mC\in\mathbb{R}^{T\times T}$ is defined as follows
\begin{align}\label{matC_m}
\begin{cases}
\mC_{ii}=\frac{(T-1)\gamma_2}{T}+\eta_i,~\forall i\in\lbrace 1,\dots,T \rbrace, \\
\mC_{ij}=-\frac{\gamma_2}{T},~\forall i,j\in\lbrace 1,\dots,T \rbrace, i\neq j.
\end{cases}
\end{align}
In addition, the matrix $\mB\in\mathbb{R}^{T\times T}$ is defined as $\mB=\mC^{-1} \circ \mL$, where $\circ$ denotes the Hadamard product, and the matrix $\mL\in\mathbb{R}^{T\times T}$ is given as
\begin{align}
\begin{cases}
\mL_{ii}=1,~\forall i\in\lbrace 1,\dots,T \rbrace, \\
\mL_{ij}=\rho,~\forall i,j\in\lbrace 1,\dots,T \rbrace,~i\neq j.
\end{cases}
\end{align}
The expectation in the cost of the loss in \eqref{det_form_m} is over the standard Gaussian random variables $H_t$ and $S_t$ and the random variable $Y_t$, which is defined as
\begin{align}
Y_t=\varphi \bigg(\frac{1}{\sqrt{\rho}} \left[S_t\sqrt{\frac{\kappa_t}{\alpha_t}} + Z_t  \sqrt{1-\frac{\kappa_t}{\alpha_t}} \ \right] \bigg),
\end{align}
where $Z_t$ is an independent standard Gaussian random variable.

\subsection{Asymptotic Predictions}
Now, we are ready to state our main theoretical predictions for the multi--task approach employed in \eqref{multi_task}. 
\begin{theorem}[General Multi--Task Analysis]\label{th_gmtask}
Suppose that the assumptions \ref{rad_fv}-\ref{highdim} are satisfied. Then,  the generalization error corresponding to the $t^{th}$ task of the general formulation in \eqref{multi_task} converges in probability as follows 
\begin{align}\label{gen_conv}
{\mathcal{E}}_{p,t,\text{test}} \xrightarrow{p \to \infty} \frac{1}{4^\vartheta} \mathbb{E}\left[ \left( \varphi(c_0 G_1) -\widehat{\varphi}(c_{1,t} G_1 + c_{2,t} G_2) \right)^2 \right],
\end{align}
where $G_1$ and $G_2$ are two independent standard Gaussian random variables.
Furthermore, $c_{1,t}$ and $c_{2,t}$ are given as
\begin{align}
c_{1,t}=q_t^\star \sqrt{\frac{\kappa_t}{\alpha_t}} ,~ \text{and} \ c_{2,t}=\sqrt{ \left(1-\frac{\kappa_t}{\alpha_t} \right) (q_t^\star)^2+(r_t^\star)^2 }.
\end{align}
The terms $r_t^\star$ and $q_{t}^\star$ are the optimal solutions of the scalar formulation given in \eqref{det_form_m}. 
\end{theorem} 
\begin{proof}
The proof of the asymptotic results stated in Theorem \ref{th_gmtask} is based on an extended version of the CGMT framework \cite{dhi21inherent}, and the theoretical results in \cite{trustsub,andersen1982,Newey94,Dabah}. 
To streamline our presentation, we postpone the details to the appendix.
\end{proof}
The result stated in Theorem \ref{th_gmtask} is a generalized version of the predictions given in Theorem \ref{lem_same}. Specifically, the results in Theorem \ref{th_gmtask} are valid for any choice of  $\lbrace \kappa_t \rbrace_{1 \leq t \leq T}$. Our analysis shows that the deterministic problem in \eqref{det_form_m} is the asymptotic limit of the multi--task formulation. Moreover, it shows that the formulation in \eqref{det_form_m} is strictly convex in the minimization variables. This proves the uniqueness of the solutions of the problem in \eqref{det_form_m}.

Next, we give a simulation example to validate the results stated in Theorem \ref{th_gmtask}. We consider the binary classification model and the squared loss function. Figure \ref{figsim5} considers two tasks. It also assumes that the training data size of the first task is two times more the training data size of the second task.  We can see from Figure \ref{figsim5} that the predictions in Theorem \ref{th_gmtask} are in excellent agreement with the actual performance of the multi--task formulation.
\begin{figure}[h!]
    \centering
    \subfigure[]{
        \includegraphics[width=0.48\linewidth]{ 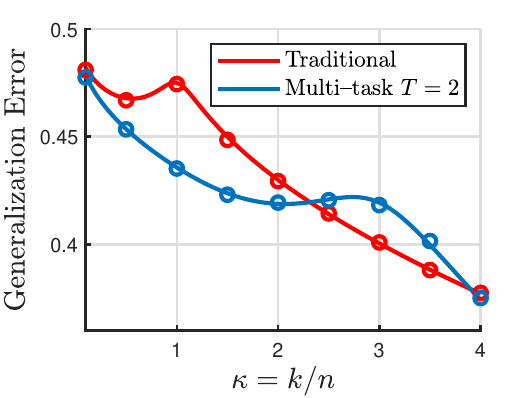}}
    \subfigure[]{
    \includegraphics[width=0.48\linewidth]{ 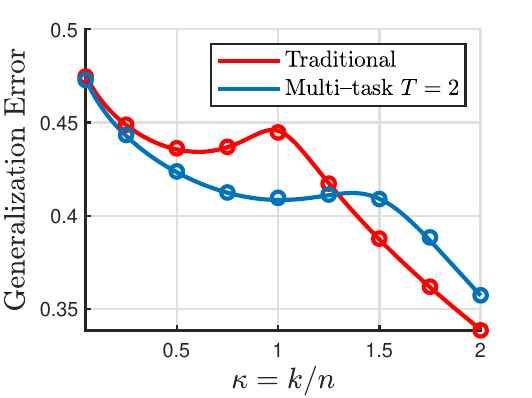}}
     
    \caption{Solid lines: Theoretical predictions in Theorem 2. Circles: numerical simulation for the multi--task formulation. We consider two tasks in the multi--task formulation. The parameters are set as follows $p=2000$, $\alpha=4$, $\rho=0.7$, $T=2$, $\gamma_1=0.005$ and $\gamma_1=1$. Moreover, we take $\alpha_1=\alpha$ and $\alpha_2=\alpha/2$. {\bf (a)} The performance of the first task. {\bf (b)} The performance of the second task.  The results are averaged over $100$ independent Monte Carlo trials. 
    }
        \label{figsim5}
\end{figure}
This provides an empirical verification of the results stated in Theorem \ref{th_gmtask}. Moreover, observe that the generalization error corresponding to the multi--task formulation exhibits the same qualitative behavior as for the symmetric formulation. Specifically, we can see that the double descent peak is postponed and the multi--task formulation improves the generalization performance for small $\kappa$.
\section{Additional Numerical Investigations}\label{sim_exp}
In this part, we provide additional simulation examples to empirically verify our theoretical predictions derived in the previous parts. 

In the first simulation example, we verify the results stated in Corollary \ref{cor_reg}. Specifically, we verify that the asymptotic performance of the separate formulation, where $R(\rho)$ is selected according to \eqref{fixed_equ}, can be precisely predicted by solving the deterministic formulation in \eqref{smulti_task_asy}.
\begin{figure}[h!]
    \centering
    \subfigure[]{
        \includegraphics[width=0.48\linewidth]{ 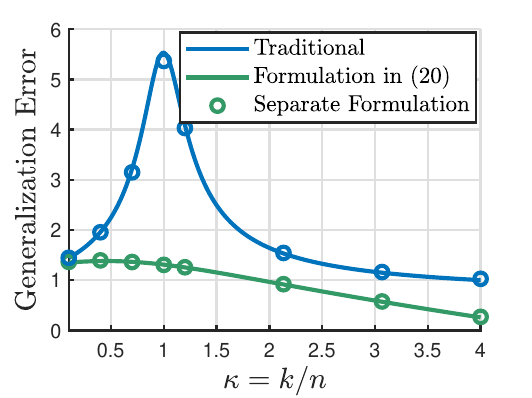}}
    \subfigure[]{
    \includegraphics[width=0.48\linewidth]{ 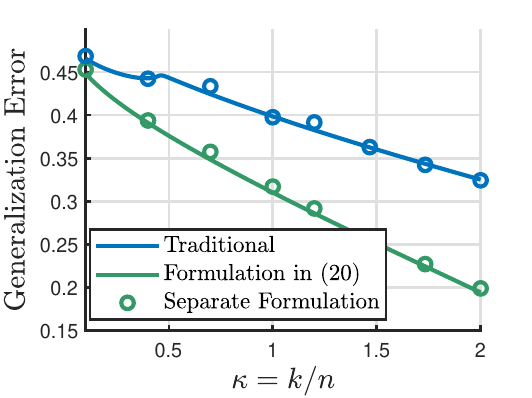}}
     
    \caption{Performance of the first task. Continuous lines: Theoretical predictions. Circles: Numerical simulations for the traditional and separate formulations. {\bf (a)} We consider the linear regression model and the squared loss. The parameters are set as follows $p=1000$, $\alpha=4$, $\gamma_1=0.01$, $\gamma_2=0.6$ and $\rho=0.75$. {\bf (b)} We consider the binary classification model and the logistic loss. The parameters are set as follows $p=1000$, $\alpha=2$, $\gamma_1=10^{-4}$, $\gamma_2=0.4$ and $\rho=0.6$. The results are averaged over $50$ independent Monte Carlo trials.
}
        \label{figsim6}
\end{figure}
Figure \ref{figsim6} considers the linear regression model with the squared loss and the binary classification model with the logistic loss. First, we can notice that the performance of the separate formulation introduced in \eqref{smulti_task} is in excellent agreement with the performance of the scalar formulation in \eqref{det_form_asy2}, even for a moderate problem dimensions. In addition, Figure \ref{figsim6} illustrates that the combination of large number of tasks significantly improves the generalization error and mitigates the double descent phenomenon. Essentially, it leads to a strictly decreasing generalization error as a function of the problem parameter $\kappa$.

In the second simulation example, we consider the binary classification model with the squared loss. Moreover, we consider four task in the formulation in \eqref{multi_task}. The first two tasks have the same training data size. Also, the third and fourth tasks have half the training data size of the first two tasks. Figure \ref{figsim7} first validates the results stated in Theorem \ref{th_gmtask}. This can be achieved by observing that they are in excellent agreement with the actual performance of \eqref{multi_task}, even for a moderate problem dimensions. Moreover, Figure \ref{figsim7} empirically studies the performance of the general multi--task formulation. Figure \ref{figsim7a} first shows that the generalization error corresponding to the first task improves as we increase the similarity measure $\rho$. Also, note that the traditional formulation is better than the multi--task formulation for a small similarity between the tasks. 
\begin{figure}[h!]
    \centering
    \subfigure[]{\label{figsim7a}
        \includegraphics[width=0.48\linewidth]{ 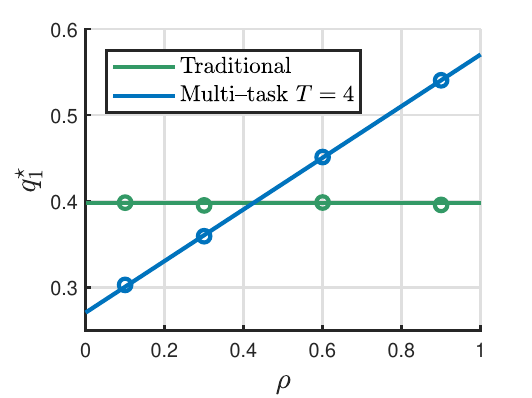}}
    \subfigure[]{\label{figsim7b}
    \includegraphics[width=0.48\linewidth]{ 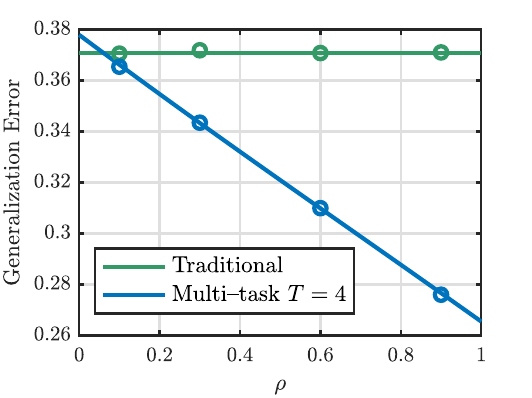}}
     
    \caption{Performance of the first task. Continuous lines: Theoretical predictions. Circles: Numerical simulations for the traditional and multi--task formulations. We consider the binary classification model and the squared loss. The parameters are set as follows $p=2000$, $\alpha_1=\alpha_2=2$, $\alpha_3=\alpha_4=1$, $\kappa_1=\kappa_2=1.5$, $\kappa_3=\kappa_4=0.75$, $\gamma_1=0.1$ and $\gamma_2=1$. {\bf (a)} The behavior of the optimal value $q_1^\star$. {\bf (b)} The behavior of the generalization error. The results are averaged over $50$ independent Monte Carlo trials.}
        \label{figsim7}
\end{figure}
Figure \ref{figsim7b} shows that the optimal value $q_1^\star$ increases as we increase the similarity measure. This suggests that the general multi--task formulation favors solutions aligned with the generative model in \eqref{mt_model}. This also suggests that the regularization properties of the general multi--task formulation exhibit the same qualitative behavior as for the symmetric formulation stated in Corollary \ref{cor_reg}. 

\section{Conclusion}\label{concd}
In this paper, we precisely analyzed a popular multi--task formulation. Specifically, we provided an exact characterization of the generalization error corresponding to the considered multi--task formulation. The predictions are based on a multivariate version of the CGMT framework. Our precise results are then used to study the regularization effects of the considered multi--task formulation. Particularly, we showed that the multi--task formulation is asymptotically equivalent to a traditional formulation with an additional regularization that favors solutions aligned with the generative model. Moreover, we empirically studied the impact of combining tasks on the generalization error. In particular, it has been empirically shown that the combination of multiple tasks postpones the double descent phenomenon and can mitigate it asymptotically.
\appendix
\section{Appendix}\label{tech_deta}
In this appendix, we provide a proof outline of the theoretical results stated in this paper. 
Our analysis is based on an extended version of the CGMT framework referred to as the multivariate CGMT. The analysis is valid under the assumptions \ref{rad_fv}-\ref{highdim}. Our approach is to prove the general results in Theorem \ref{th_gmtask}. Then, specialize the results in Theorem \ref{th_gmtask} to the settings considered in Theorem \ref{lem_same}. 
\subsection{Multivariate Convex Gaussian Min-max Theorem}
To rigorously prove the technical results stated in Theorem \ref{th_gmtask}, we use an extended version of the CGMT framework, that is called the multivariate convex Gaussian min-max theorem (MCGMT) \cite{dhi21inherent}. The MCGMT replaces the high--dimensional analysis of a generally hard primary problem with a simpler formulation. In this paper, we consider primary problems of the form 
\begin{equation}\label{MPO}
\Phi_p=\min\limits_{\vw \in \mathcal{S}_{\vw}} \max\limits_{\vu\in\mathcal{S}_{\vu}} \sum_{t=1}^{T} \vu_t^\top \mG_t \vw_t + \Upsilon(\vw,\vu),
\end{equation}
where $\vu_t \in \mathbb{R}^{n_t}$ and $\vw_t \in \mathbb{R}^{k}$ are optimization variables, and $\mG_t \in \mathbb{R}^{n_t\times k}$ has independent standard Gaussian random components, for any $t\in\lbrace 1,\dots,T \rbrace$. Additionally, the vectors $\vw$ and $\vu$ are formed by the concatenation of the vectors $\lbrace \vw_t \rbrace_{t=1}^{T}$ and $\lbrace \vu_t \rbrace_{t=1}^{T}$, respectively. We refer to the formulation in \eqref{MPO} as the multivariate primary optimization (MPO). 
Then, the corresponding multivariate auxiliary optimization (MAO) is given by
\begin{align}\label{MAO}
\phi_p=\min\limits_{\vw \in \mathcal{S}_{\vw}} \max\limits_{\vu\in\mathcal{S}_{\vu}} &\sum_{t=1}^{T} \norm{\vu_t} \vg_t^\top \vw_t + \sum_{t=1}^{T} \norm{\vw_t} \vh_t^\top \vu_t + \Upsilon(\vw,\vu),
\end{align}
where $\vg_t\in\mathbb{R}^{k}$ and $\vh_t\in\mathbb{R}^{n_t}$ are independent standard Gaussian random vectors, for any $t\in\lbrace 1,\dots,T \rbrace$. Here, we assume that $\mG_t \in\mathbb{R}^{ n_t \times  k}$, $\vg_t \in \mathbb{R}^{ k}$, and $\vh_t\in\mathbb{R}^{ n_t}$ are all independent of each other, the feasibility sets $\mathcal{S}_{\vw}\subset\R^{ Tk}$ and $\mathcal{S}_{\vu}\subset\R^{ n}$ are convex and compact, and the function $\Upsilon: \mathbb{R}^{Tk} \times \mathbb{R}^{ {n}} \to \mathbb{R}$ is continuous \emph{convex-concave} on $\mathcal{S}_{\vw}\times \mathcal{S}_{\vu}$, where $n = \sum_{t=1}^T n_t$.
The following theorem shows that the optimization problems in \eqref{MPO} and \eqref{MAO} are equivalent in the large system limit. 
\begin{theorem}[MCGMT \cite{dhi21inherent}]\label{mcgmt}
For any fixed $T \geq 1$, define the open set $\mathcal{S}_{p}$.
Moreover, define the set $\mathcal{S}^c_{p}={\mathcal{S}}_{\vw} \setminus \mathcal{S}_{p}$. Let $\phi_p$ and $\phi^c_{p}$ be the optimal cost values of the MAO formulation in \eqref{MAO} with feasibility sets ${\mathcal{S}}_{\vw}$ and $\mathcal{S}^c_{p}$, respectively. Assume that the following properties are all satisfied
\begin{enumerate}
\item There exists a constant $\phi$ such that the optimal cost $\phi_p$ converges in probability to $\phi$ as $p$ goes to $+\infty$.
\item There exists a constant $\phi^c$ such that the optimal cost $\phi^c_{p}$ converges in probability to $\phi^c$ as $p$ goes to $+\infty$.
\item There exists a positive constant $\zeta>0$ such that $\phi^c \geq \phi+\zeta$.
\end{enumerate}
Then, the following convergence in probability holds
\begin{equation}
\abs{\Phi_p -\phi_p} \overset{p\to\infty}{\longrightarrow} 0, \ \ \text{and} \ \ \mathbb{P}( \widehat{\vw}_{p} \in \mathcal{S}_{p} )  \overset{p\to\infty}{\longrightarrow} 1,\nonumber
\end{equation}
where $\Phi_p,$ and $\widehat{\vw}_{p}$ are the optimal cost and the optimal solution of the MPO formulation in \eqref{MPO}.
\end{theorem}
The above theorem allows us to analyze the generally easy MAO formulation given in \eqref{MAO} to infer asymptotic properties of the generally hard MPO problem in \eqref{MPO}.

\subsection{Sharp Asymptotic Analysis of the Multi--Task Formulation}
In this part, we provide a sharp asymptotic analysis of the multi--task formulation given in \eqref{multi_task}. Particularly, we use the MCGMT to sharply analyze the following optimization problem
\begin{align}\label{form_new_ana}
\min_{\vw_t\in\mathbb{R}^k}&~\sum_{t=1}^{T} \frac{1}{n_t} \sum_{i=1}^{n_t} \ell \left(y_{t,i};\vb_{t,i}^\top \vw_t \right)+\frac{\gamma_1}{2} \sum_{t=1}^{T} \norm{\vw_t}^2\nonumber \\ &+\frac{\gamma_2}{2} \sum_{t=1}^{T} \norm{ \vw_t-\bar{\vw} }^2,
\end{align}
where $\bar{\vw}$ denotes the average of the optimization vectors $\lbrace \vw_t \rbrace_{t=1}^{T}$, i.e., $\bar{\vw}=\frac{1}{T}\sum_{t=1}^{T} {\vw_t}$.
Our first objective is to express the optimization problem in \eqref{form_new_ana} in the form of a MPO formulation in \eqref{MPO}. Then, apply the MCGMT framework to formulate the corresponding MAO formulation. The final step is to study the asymptotic properties of the obtained MAO.

\subsubsection{Formulating the MAO Problem}
The first step to obtain a multivariate auxiliary formulation is to formulate our optimization problem in the form of the MPO in \eqref{MPO}. To this end, we start by introducing additional optimization variables as follows
\begin{align}\label{form_ana_v1}
\min_{\vw_t\in\mathbb{R}^k}\max_{\vu_t\in\mathbb{R}^{n_t}}&~\sum_{t=1}^{T} \frac{1}{n_t} \sum_{i=1}^{n_t} u_{t,i} \vb_{t,i}^\top \vw_t - \sum_{t=1}^{T} \frac{1}{n_t} \sum_{i=1}^{n_t} \ell^\star \left(y_{t,i};u_{t,i}\right)
\nonumber \\ &+\frac{\gamma_1}{2} \sum_{t=1}^{T} \norm{\vw_t}^2+\frac{\gamma_2}{2} \sum_{t=1}^{T} \norm{ \vw_t-\bar{\vw} }^2.
\end{align}
Here, the function $\ell^\star(y;.)$ denotes the convex conjugate of the loss function $\ell(y;.)$. Define the matrix $\mB_t\in\mathbb{R}^{n_t\times k}$ as the concatenation of the vectors $\lbrace \vb_{t,i}^\top \rbrace_{1\leq i \leq n_t}$. The multivariate version of the CGMT assumes that the feasibility sets of the MPO are convex and compact. Although these properties are not trivial in our case, one can follow the approaches in \cite{chris:151,ouss19,dhifallah2020} to prove that the optimal solutions of the formulation in \eqref{form_ana_v1} belong to convex and compact sets, asymptotically. This implies that one can equivalently formulate the problem in \eqref{form_ana_v1} with convex and compact feasibility sets. In the rest of this paper, we only consider convex feasibility sets where the compactness is assumed implicitly. 
Note that the labels $\lbrace y_{t,i} \rbrace_{1 \leq i \leq n_t}$ depend on the matrix $\mB_t$, therefore, we cannot directly apply the MCGMT. To overcome this issue, we decompose the matrix $\mB_t$ without changing its statistics as follows 
\begin{align}\label{b_dist}
\mB_t&=\mB_t \bar{\vxi}_{ts} \bar{\vxi}_{ts}^\top+\mB_t\mK^\perp_{t} = \vs_t \bar{\vxi}_{ts}^\top+\mG_t\mK^\perp_{t},
\end{align}
where the elements of $\vs_t\in\mathbb{R}^{n_t}$ and $\mG_t\in\mathbb{R}^{n_t\times k}$ are drawn independently from a standard Gaussian distribution, while $\vs_t$ and $\mG_t$ are independent of each other. 
Furthermore, ${\vxi}_{ts}$ denotes the entries of the vector $\vxi_t$ with index in the set $\mathcal{S}$ and $\bar{\vxi}_{ts}$ is the normalized version of ${\vxi}_{ts}$. 
Also, $\mK^\perp_{t}\in\mathbb{R}^{k\times k}$ represents the projection matrix onto the orthogonal complement of the space spanned by the vector $\bar{\vxi}_{ts}$. 
Note that the result in \eqref{b_dist} is equality in distribution. Then, one can formulate the optimization problem \eqref{form_ana_v2} as follows
\begin{align}\label{form_ana_v2}
\min_{\vw_t\in\mathbb{R}^k}\max_{\vu_t\in\mathbb{R}^{n_t}}&~\sum_{t=1}^{T} \frac{1}{n_t} \vu_t^\top \mG_t \mK^\perp_{t} \vw_t + \sum_{t=1}^{T} \frac{1}{n_t} \vu_t^\top \vs_t \bar{{\vxi}}_{ts}^\top \vw_t \nonumber\\
&+\frac{\gamma_1}{2} \sum_{t=1}^{T} \norm{\vw_t}^2 - \sum_{t=1}^{T} \frac{1}{n_t} \sum_{i=1}^{n_t} \ell^\star \left(y_{t,i};u_{t,i}\right) \nonumber\\
&+\frac{\gamma_2}{2} \sum_{t=1}^{T} \norm{ \vw_t-\bar{\vw} }^2.
\end{align}
Note that the formulation in \eqref{form_ana_v2} is in the form of the MPO problem introduced in \eqref{MPO}. Moreover, one can see that the convexity assumption in the MCGMT framework is satisfied. Then, the corresponding MAO formulation can be expressed as follows
\begin{align}\label{form_ana_v3}
&\min_{\vw_t\in\mathbb{R}^k}\max_{\vu_t\in\mathbb{R}^{n_t}}~ \sum_{t=1}^{T} \frac{\norm{\vu_t}}{n_t} \vg_t^\top \mK^\perp_{t} \vw_t+\sum_{t=1}^{T} \frac{1}{n_t}\norm{\mK^\perp_{t} \vw_t} \vh_t^\top \vu_t \nonumber\\
&+\frac{\gamma_2}{2} \sum_{t=1}^{T} \norm{ \vw_t-\bar{\vw} }^2  + \sum_{t=1}^{T} \frac{\vu_t^\top \vs_t \bar{\vxi}_{ts}^\top \vw_t}{n_t} \nonumber\\
&- \sum_{t=1}^{T} \frac{1}{n_t} \sum_{i=1}^{n_t} \ell^\star \left(y_{t,i};u_{t,i}\right)+\frac{\gamma_1}{2} \sum_{t=1}^{T} \norm{\vw_t}^2.
\end{align}
Here, the vectors $\vg_t\in\mathbb{R}^k$ and $\vh_t\in\mathbb{R}^{n_t}$ have components independently drawn from a standard Gaussian distribution. Next, we focus our attention on expressing the MAO formulation in \eqref{form_ana_v3} in terms of scalar variables, then, studying its asymptotic properties.

\subsubsection{Simplifying the MAO Formulation}
We start the analysis of the auxiliary formulation by decomposing the optimization variable $\vw_t$ as follows
\begin{align}\label{wt_decomp}
\vw_t=(\bar{\vxi}_{ts}^\top \vw_t) \bar{\vxi}_{ts} + \mP_{ts} \vr_t,
\end{align}
where $\vr_t\in\mathbb{R}^{k-1}$ is a free vector, and $\mP_{ts}\in\mathbb{R}^{k\times(k-1)}$ is formed by an orthonormal subspace orthogonal to the vector $\vxi_{ts}$. In addition, define the scalar $q_t$ as follows
\begin{align}
q_t=\bar{\vxi}_{ts}^\top \vw_t.
\end{align}
Now, we fix $q_t$ and the norm of $r_t=\norm{\vr_t}$ in the formulation in \eqref{form_ana_v3}. Moreover, we solve over the direction of the optimization vector $\vw_t$. This optimization problem can be formulated as follows
\begin{align}
\min_{\substack{\norm{\vw_t}^2=q_t^2+r_t^2\\q_t=\bar{\vxi}_{ts}^\top \vw_t}}& \sum_{t=1}^{T} \frac{\norm{\vu_t}}{n_t} \vg_t^\top \vw_t
+\frac{\gamma_2}{2} \sum_{t=1}^{T} \norm{ \vw_t-\bar{\vw} }^2 ,
\label{min_wt}
\end{align}
where we drop the terms that are independent of the direction of the vector $\vw_t$.  
In addition, we use the fact that $\frac{1}{n_t} \vg_t^\top\mK^\perp_{t} \vw_t$ is asymptotically equivalent to $\frac{1}{n_t}\vg_t^\top \vw_t$. Note that the optimization in \eqref{min_wt} is not convex due to the norm equality constraint. However, one can use an extended version of the approach proposed in \cite{trustsub} to solve \eqref{min_wt}. Specifically, the formulation in \eqref{min_wt} can be rewritten as
\begin{align}
&\max_{\lambda_t,\eta_t\in\mathcal{F}_t}\min_{{\vw_t}}~\sum_{t=1}^{T} \frac{\norm{\vu_t}}{n_t} \vg_t^\top \vw_t
+\frac{\gamma_2}{2} \sum_{t=1}^{T} \norm{ \vw_t-\bar{\vw} }^2 \nonumber\\
&+\sum_{t=1}^{T}\lambda_t \big(\bar{\vxi}_{ts}^\top \vw_t-q_t \big) +\frac{1}{2}\sum_{t=1}^{T}\eta_t \big(\norm{\vw_t}^2-q_t^2-r_t^2 \big),
\label{min_wt_lag}
\end{align}
{where the optimization variables $\eta_t$ satisfy the regularity conditions in $\mathcal{F}_t$ that will be defined later.}
Here, the variables $\lambda_t$ and $\eta_t$ are the Lagrange multipliers. Next, define the vector $\vw\in\mathbb{R}^{kT}$ to be the concatenation of the optimization vectors $\vw_t$, i.e., $\vw=[\vw_1^\top,\cdots,\vw_T^\top]^\top$. Then, the optimization problem expressed in \eqref{min_wt_lag} can be compactly formulated as follows
\begin{align}\label{Eq:hh}
\max_{\substack{\lambda_t,\eta_t\\ \mC_p \succ 0 }}\min_{{\vw}}~\frac{1}{2} \vw^\top \mC_p \vw+\vb^\top\vw-\frac{1}{2}\sum_{t=1}^{T}\eta_t(q_t^2+r_t^2)-\sum_{t=1}^{T}\lambda_t q_t.
\end{align}
Here, the matrix $\mC_p\in\mathbb{R}^{kT\times kT}$ is defined as follows $\mC_p=\gamma_2\mSig+{\mathbf{\Delta}}$, where the matrix $\mathbf{\Delta}$ is a weighted block diagonal identity matrix, i.e., ${\mathbf{\Delta}}=\text{diag}(\eta_1\mI_k,\cdots,\eta_T\mI_k)$. Additionally, the matrix $\mSig$ is defined as follows $\mSig=\sum_{t=1}^T\mSig_t^\top \mSig_t$. Furthermore, the matrix $\mSig_t\in \mathbb{R}^{k\times kT}$ is expressed as 
\begin{align}
\mSig_t=\left[-\frac{1}{T}\mI_k,\cdots \frac{T-1}{T}\mI_k,\cdots,-\frac{1}{T}\mI_k \right],
\end{align}
where the matrix $\frac{T-1}{T} \mI_k$ is in the $t^{th}$ block.
Here, the positive--definiteness constraint in the above formulation represents the regularity conditions introduced in \eqref{min_wt_lag}, i.e., $\mathcal{F}_t = \{ \eta_t \in \mathbb{R}: \mC_p \succ 0  \}$. In the above formulation, the vector $\vb\in \mathbb{R}^{kT\times 1}$ is defined as follows
\begin{align}
\vb^\top=\left[\frac{\|\vu_1\|}{n_1}\vg_1^\top+\lambda_1 \bar{\vxi}_{1s}^\top, \cdots,\frac{\|\vu_T\|}{n_T}\vg_T^\top+\lambda_T \bar{\vxi}_{Ts}^\top \right].
\end{align}
Now, we are in a position to simplify the formulation in \eqref{Eq:hh} over the optimization vector $\vw$. Note that the formulation is now convex in $\vw$. Moreover, it can be simplified as follows
\begin{align}\label{min_wt_sc}
\max_{\substack{\lambda_t,\eta_t\\ \mC_p \succ 0 }}&- \frac{1}{2} \vb^\top\mC_p^{-1} \vb-\frac{1}{2}\sum_{t=1}^{T}\eta_t(q_t^2+r_t^2)-\sum_{t=1}^{T}\lambda_tq_t.
\end{align}
Note that the above steps simplify the optimization problem in \eqref{Eq:hh} to a scalar formulation as given in \eqref{min_wt_sc}. This implies that the MAO formulation obtained in \eqref{form_ana_v3} can be equivalently reformulated as follows
\begin{align}\label{form_ana_v4}
\min_{\vq,\vr\geq 0}&\max_{\substack{\veta,\vlambda \\ \mC_p \succ 0,\vu_t\in\mathbb{R}^{n_t}}}~ \sum_{t=1}^{T} \frac{1}{n_t}r_t \vh_t^\top \vu_t   
+ \sum_{t=1}^{T} \frac{q_t\vu_t^\top \vs_t }{n_t}- \frac{1}{2}\vz^\top\mC_p^{-1}\vz \nonumber\\
&- \sum_{t=1}^{T} \frac{1}{n_t} \sum_{i=1}^{n_t} \ell^\star \left(y_{t,i};u_{t,i}\right)-\sum_{t=1}^{T}\frac{\|\vu_t\|^2}{2 n_t}V_{p,t}(\veta)\nonumber\\
&-\vlambda^\top \vq +\frac{1}{2}\sum_{t=1}^{T}(\gamma_1-\eta_t)(q_t^2+r_t^2).
\end{align}
Here, the vector $\vq \in\mathbb{R}^T$ and $\vr\in\mathbb{R}^T$ are formed by the concatenation of $\lbrace  q_t \rbrace_{1\leq t \leq T}$ and $\lbrace  r_t\rbrace_{1\leq t \leq T}$, respectively. Moreover, the vector $\veta\in\mathbb{R}^T$ and $\vlambda\in\mathbb{R}^T$ are formed by the concatenation of $\lbrace  \eta_t\rbrace_{1\leq t \leq T}$ and $\lbrace  \lambda_t\rbrace_{1\leq t \leq T}$, respectively. Also, the function $V_{p,t}(\cdot)$ can be expressed as follows
\begin{align}
V_{p,t}(\veta)=\frac{1}{n_t}\vg_t^\top\mC_{p,tt}^{-1}\vg_t.
\end{align}
Additionally, the vector $\vz\in\mathbb{R}^{kT}$ is defined as $\vz=[\lambda_1\bar{\vxi}_{1s}^\top,...,\lambda_T\bar{\vxi}_{Ts}^\top]^\top$,  and $\mC_{p,tt}^{-1}\in\mathbb{R}^{k\times k}$ denotes the $t^{th}$ diagonal block of the matrix $\mC_p^{-1}$. Here, the formulation in \eqref{form_ana_v4} is obtained after dropping terms that converge in probability to zero. This result can be justified by proving that these functions also converge uniformly in probability to the same limit \cite{dhifallah2020}. It remains to simplify the formulation in \eqref{form_ana_v4} over the optimization vector $\lbrace \vu_t \rbrace_{1\leq t \leq T}$. To this end, using the property in \cite[Example 11.26]{var_prog}, we have the following equivalent representation
\begin{align}
&\max_{\substack{\vu_t}}  \frac{1}{n_t}(r_t \vh_t+q_t\vs_t)^\top \vu_t   -\frac{\|\vu_t\|^2}{2n_t}V_{p,t}(\veta) -  \frac{1}{n_t} \sum_{i=1}^{n_t} \ell^\star \left(y_{t,i};u_{t,i}\right)\nonumber\\&=\frac{1}{n_t}\sum_{i=1}^{n_t} \mathcal{M}_{\ell(y_{t,i};.)}\big(r_t h_{t,i}+q_t s_{t,i};V_{p,t}(\veta)\big).
\end{align}
Note that the above equality transforms a vector optimization to a sum of separable scalar formulations. This implies that the MAO formulation expressed in \eqref{form_ana_v4} can be equivalently formulated as follows
\begin{align}
\label{AO_117}
&\min_{\vq,\vr\geq 0}\max_{\substack{\veta,\vlambda\\\mC_p \succ 0}}~\frac{1}{2}\sum_{t=1}^{T}(\gamma_1-\eta_t)(q_t^2+r_t^2)- \vlambda^\top \vq- \frac{1}{2}\vz^\top\mC_p^{-1}\vz \nonumber\\
&+\sum_{t=1}^T\frac{1}{n_t}\sum_{i=1}^{n_t} \mathcal{M}_{\ell(y_{t,i};.)}\big(r_t h_{t,i}+q_t s_{t,i};V_{p,t}(\veta)\big).
\end{align}
Note that our MAO problem is now expressed in terms of scalar optimization variables. Here, our final step is to solve over the variable $\vlambda$. 
Then, the MAO formulation obtained in \eqref{AO_117} can be rewritten as
\begin{align}\label{com_form}
&\min_{\vq,\vr \geq 0}\max_{\substack{\veta,\vlambda\\\mC_p \succ 0}}~\frac{1}{2}\sum_{t=1}^{T}(\gamma_1-\eta_t)(q_t^2+r_t^2)-\vlambda^\top\vq -\frac{1}{2}\vlambda^\top\mB_p\vlambda  \nonumber\\
&+\sum_{t=1}^T\frac{1}{n_t}\sum_{i=1}^{n_t} \mathcal{M}_{\ell(y_{t,i};.)}\big(r_t h_{t,i}+q_t s_{t,i};V_{p,t}(\veta)\big).
\end{align}
In the above, the matrix $\mB_p\in\mathbb{R}^{T\times T}$ has $(i,j)^{th}$ component defined as 
\begin{align}\label{matB}
B_{ij}= \bar{\vxi}_{is}^\top\mC^{-1}_{p,ij}\bar{\vxi}_{js},~\forall i,j\in\lbrace 1,\dots,T \rbrace,
\end{align}
where $\mC^{-1}_{p,ij} \in \mathbb{R}^{k\times k}$ represents the $(i,j)^{th}$  block of the matrix $\mC_p^{-1}$. Using the compact formulation in \eqref{com_form}, one can easily solve the optimization over the vector $\vlambda$.
Particularly, the multivariate auxiliary formulation expressed in \eqref{AO_117} can be equivalently formulated as
\begin{align}\label{scalar_form}
\min_{\vq,\vr \geq 0}&\max_{\substack{\veta,\mC_p \succ 0}}\frac{1}{2}\sum_{t=1}^{T}(\gamma_1-\eta_t)(q_t^2+r_t^2)+\frac{1}{2}\vq^\top\mB_p^{-1}\vq 
\nonumber\\
&+\sum_{t=1}^T\frac{1}{n_t}\sum_{i=1}^{n_t} \mathcal{M}_{\ell(y_{t,i};.)}\big(r_t h_{t,i}+q_t s_{t,i};V_{p,t}(\veta)\big).
\end{align}
Note that the above analysis expresses the multivariate auxiliary formulation in \eqref{form_ana_v3} in terms of scalar variables as given in \eqref{scalar_form}. Then, it remains to study the asymptotic properties of the formulation in \eqref{scalar_form}. We refer to this problem as the scalar formulation.

\subsubsection{Asymptotic Analysis of the Scalar Formulation} 
In this part, we study the asymptotic properties of the scalar formulation expressed in \eqref{scalar_form}. Note that the matrix $\mC_p$ is formed by weighted block identity matrices. This means that the spectrum of the matrix $\mC_p$ can be fully characterized by analyzing the matrix $\mC\in\mathbb{R}^{T\times T}$ defined as follows
\begin{align}\label{matC}
\begin{cases}
&\mC_{ii}=\frac{(T-1)\gamma_2}{T}+\eta_i,~\forall i\in\lbrace 1,\dots,T \rbrace, \\
&\mC_{ij}=-\frac{\gamma_2}{T},~\forall i,j\in\lbrace 1,\dots,T \rbrace, i\neq j.
\end{cases}
\end{align}
Given the form of the matrix $\mC_p$, the matrix $\mC^{-1}_p$ is formed by weighted block identity matrices. The weights can be obtained by computing the inverse of the matrix $\mC$ defined above. Then, the components of the matrix $\mB_p$ defined in \eqref{matB} converges in probability to the components of the matrix $\mB\in\mathbb{R}^{T\times T}$ defined as $\mB=\mC^{-1} \circ \mL$, where $\circ$ denotes the Hadamard product, and the matrix $\mL\in\mathbb{R}^{T\times T}$ is defined as follows
\begin{align}
\begin{cases}
&\mL_{ii}=1,~\forall i\in\lbrace 1,\dots,T \rbrace, \\
&\mL_{ij}=\frac{1}{1+\sigma^2},~\forall i,j\in\lbrace 1,\dots,T \rbrace,~i\neq j .
\end{cases}
\end{align}
Based on the asymptotic results proven in \cite{Dabah}{}, the sequence of random function $V_{p,t}(\cdot)$ converges in probability as follows
\begin{align}\label{fVt}
V_{p,t}(\veta) \xrightarrow{p \to \infty} V_t(\veta) = \kappa_t C_{tt}^{-1}.
\end{align}
Here, the scalar $C_{tt}^{-1}$ denotes the $t^{th}$ diagonal element of the matrix $\mC^{-1}$ defined in \eqref{matC}.
Additionally, using the weak law of large numbers, one can show that the empirical average of the Moreau envelope function in \eqref{scalar_form} converges to the following deterministic function
\begin{align}
{F}_t(q_t,r_t,\veta)=\mathbb{E} \left[ \mathcal{M}_{\ell(Y_{t};.)} \left(r_t H_{t}+q_t S_{t};V_{t}(\veta) \right) \right],
\end{align}
where the expectation is over the standard Gaussian random variables $H_t$, $S_t$, and the random variable $Y_t$.
Also, the function $V_t(\cdot)$ is the asymptotic function defined in \eqref{fVt}. Additionally, the random variable $Y_t$ is expressed as
\begin{align}
Y_t=\varphi \bigg(\sqrt{1+\sigma^2} \left[ S_t \sqrt{\frac{\kappa_t}{\alpha_t}} +Z_t \sqrt{1-\frac{\kappa_t}{\alpha_t}} \ \right] \bigg),
\end{align}
where $S_t$ and $Z_t$ are two independent standard Gaussian random variables. 
The above analysis shows that the cost function of the scalar version of the auxiliary problem defined in \eqref{scalar_form} converges in probability to the cost function of the following deterministic formulation
\begin{align}\label{det_form}
\min_{\vq,\vr \geq 0}\max_{\substack{\veta,\mC \succ 0}}&~\frac{1}{2}\sum_{t=1}^{T}(\gamma_1-\eta_t)(q_t^2+r_t^2)+\frac{1}{2}\vq^\top\mB^{-1}\vq 
\nonumber\\
&+\sum_{t=1}^T\mathbb{E} \left[ \mathcal{M}_{\ell(Y_{t};.)}(r_t H_{t}+q_t S_{t};V_{t}(\veta)) \right].
\end{align}
A first observation is that the deterministic problem in \eqref{det_form} is not separable over the $T$ tasks given that the matrix $\mB$ is not a diagonal matrix. Now that we have obtained the asymptotic scalar formulation, it remains to study the asymptotic behavior of the generalization error corresponding to each task.

\subsubsection{Asymptotic Characterization of the Generalization Error}
In this part, we study the asymptotic properties of the generalization error corresponding to the multi--task approach employed in \eqref{multi_task}. The generalization error corresponding to the $t^{th}$ task can be expressed as follows
\begin{align}\label{testerr_ana}
\mathcal{E}_{p,t,\text{test}}&=\frac{1}{4^\vartheta} \mathbb{E}\left[ \left( \varphi(\va_{t,\text{new}}^\top\vxi_t) -\widehat{\varphi}(\widehat{\vbeta}_t^\top \va_{t,\text{new}}) \right)^2 \right],
\end{align}
where $\va_{t,\text{new}}$ is a new test input vector corresponding to the $t^{th}$ task, and $\widehat{\vbeta}_t$ is as defined in \eqref{beta_hat}.
Now, define the following two random variables
\begin{align}
\nu_1=  \va_{t,\text{new}}^\top\vxi_t,~\text{and}~\nu_2=  \widehat{\vbeta}_t^\top \va_{t,\text{new}}.\nonumber
\end{align}
For a given vectors $\widehat{\vw}_t$ and $\vxi_t$, note that the random variables $\nu_1$ and $\nu_2$ have a bivaraite Gaussian distribution with a zero mean vector and a covariance matrix given as follows
\begin{align}\label{cov_mtx}
\mGm_p=\begin{bmatrix}
\norm{\vxi_t}^2 &  \vxi_t^\top \widehat{\vbeta}_t \\
  \vxi_t^\top\widehat{\vbeta}_t &  \norm{\widehat{\vbeta}_t}^2
\end{bmatrix}.
\end{align}
To precisely analyze the asymptotic behavior of the generalization error, it suffices to analyze the properties of the covariance matrix $\mGm_p$. Define the random variables $\widehat{q}_{p,t}^\star$ and $\widehat{r}_{p,t}^\star$ for the $t^{th}$ task as follows
\begin{align}\label{opt_po}
&\widehat{q}_{p,t}^\star=\bar{\vxi}_{ts}^\top \widehat{\vw}_t,~\text{and}~\widehat{r}_{p,t}^\star=\norm{\mP_{ts}^\top \widehat{\vw}_t },
\end{align}
where the matrix $\mP_{ts}$ is defined in \eqref{wt_decomp}. The decomposition in \eqref{opt_po} shows that the covariance matrix $\mGm_p$ given in \eqref{cov_mtx} can be expressed as follows
\begin{align}
\mGm_p =
\begin{bmatrix}
1+\sigma^2 & \sqrt{1+\sigma^2} \sqrt{\kappa_t/\alpha_t} \widehat{q}^{\star}_{p,t} \\
 \sqrt{1+\sigma^2} \sqrt{\kappa_t/\alpha_t} \widehat{q}^{\star}_{p,t} &  (\widehat{q}^{\star}_{p,t})^2+(\widehat{r}^{\star}_{p,t})^2
\end{bmatrix}.\nonumber
\end{align}
Therefore, following the same lines as in \cite[Appendix B]{dhi21inherent}, the generalization error corresponding to the $t^{th}$ task can be expressed as
\begin{align}\label{GG_A}
{\mathcal{E}}_{p,t,\text{test}} = \frac{1}{4^\vartheta} \mathbb{E}\left[ \bigg( \varphi( c_0 G_1) -\widehat{\varphi}(\widetilde c_{1,t} G_1 +\widetilde c_{2,t} G_2) \bigg)^2 \right].
\end{align}
Here, $G_1$ and $G_2$ are two independent standard Gaussian random variables. Additionally, $ c_0$, $\widetilde c_{1,t}$ and $\widetilde c_{2,t}$ are constants defined as follows
\begin{align}
 &c_0=\frac{1}{\sqrt{ \rho}},~\widetilde c_{1,t}=\widehat{q}^{\star}_{p,t} \sqrt{\frac{\kappa_t}{\alpha_t}}  ,~ \text{and} \nonumber\\
& \widetilde c_{2,t}=\sqrt{ \left(1-\frac{\kappa_t}{\alpha_t}\right) (\widehat{q}^{\star}_{p,t})^2+(\widehat{r}^{\star}_{p,t})^2 }.
\end{align}
Hence, to study the asymptotic properties of the generalization error, it suffices to study the asymptotic properties of the random quantities $\widehat{q}_{p,t}^\star$ and $\widehat{r}_{p,t}^\star$. The following Lemma shows that the sequence of random variables $\widehat{q}_{p,t}^\star$ and $\widehat{r}_{p,t}^\star$ concentrates in the large system limit.
\begin{lemma}[Consistency of the Multi--Task Formulation]\label{optm_sl_conv}
The random quantities $\widehat{q}_{p,t}^\star$ and $\widehat{r}_{p,t}^\star$ satisfies the following asymptotic properties
\begin{align}
\widehat{q}^\star_{p,t} \xrightarrow{p \to \infty} q_t^\star,~\text{and}~\widehat{r}_{p,t}^\star \xrightarrow{p \to \infty} r_t^\star,\nonumber
\end{align}
where $q_t^\star$ and $r_t^\star$ are the optimal solutions of the deterministic formulation introduced in \eqref{det_form}.
\end{lemma}
The proof of Lemma \ref{optm_sl_conv} follows using the same steps of the theoretical result in \cite[Proposition 5]{dhifallah2020}. Specifically, the proof uses the results proved in \cite[Theorem 2.1]{Newey94} which requires the uniform convergence, the strict convexity of the cost function and the compactness of the feasibility sets of the deterministic formulation in \eqref{det_form}. The uniform convergence can be verified using the result in \cite[Theorem II.1]{andersen1982}, the compactness of the feasibility sets and the strict convexity properties of \eqref{det_form}. Based on the above analysis, these assumptions are all satisfied for the formulations in \eqref{scalar_form} and \eqref{det_form}.


\subsubsection{Specializing the Results to Theorem~\ref{lem_same}}
Now, that we have proven Theorem~\ref{th_gmtask}, we turn our attention towards Theorem~\ref{lem_same}, which is a special case of Theorem~\ref{th_gmtask}, with $\alpha_t=\alpha, \forall t\in\lbrace 1,\dots,T \rbrace$.
Under this condition, it can be easily seen that we have symmetric optimization problems, i.e., $q_t =q, r_t=r, ~\text{and}~ \eta_t =\eta,~ \forall t$. Then, one can simplify the expressions in \eqref{det_form}, and \eqref{GG_A}, using straightforward algebraic manipulations to arrive at the results in \eqref{det_form_asy}, and \eqref{gen_conv}, respectively.


\bibliographystyle{IEEEbib}
\bibliography{Main_Arxiv1}
\end{document}